\documentclass{article}

\usepackage{mathtools}
\usepackage{bbm}
\usepackage{hyperref}
\usepackage{color}
\usepackage[square,numbers]{natbib} % Biblio.
\usepackage{appendix} % Appendix
\usepackage[preprint]{neurips_2020}
\usepackage{caption}
\usepackage[utf8]{inputenc} % allow utf-8 input
\usepackage{hyperref}       % hyperlinks
\usepackage{url}            % simple URL typesetting
\usepackage{booktabs}       % professional-quality tables
\usepackage{amsfonts}       % blackboard math symbols
\usepackage{nicefrac}       % compact symbols for 1/2, etc.
\usepackage{microtype}      % microtypography
\usepackage{float}
\usepackage{csquotes}

\usepackage{wrapfig}
\usepackage{tikz}
\usetikzlibrary {arrows.meta}
\usepackage{thmtools, thm-restate}
\usepackage{subcaption}

\usepackage{algorithm} % Algorithms.
\usepackage{algpseudocode}

\usepackage{graphicx}
\graphicspath{{.}{figs/}}
\usepackage{booktabs}

\usepackage{amsmath,amsfonts,amssymb}
\usepackage{mathletters} % Shortcuts
\usepackage{amsthm}
\newcommand{\beq}{\begin{equation}}
\newcommand{\eeq}{\end{equation}}
\newcommand{\baln}{\begin{align*}}
\newcommand{\ealn}{\end{align*}}
\newcommand{\beqn}{\begin{equation*}}
\newcommand{\eeqn}{\end{equation*}}
\newcommand{\beqa}{\begin{eqnarray}}
\newcommand{\eeqa}{\end{eqnarray}}
\newcommand{\beqan}{\begin{eqnarray*}}
\newcommand{\eeqan}{\end{eqnarray*}}
\newtheorem*{theorem*}{Theorem}
\newtheorem{proposition}{Proposition}
\newtheorem{theorem}{Theorem}
\newtheorem{lemma}{Lemma}
\newtheorem{corollary}{Corollary}

\newtheorem{remark}{Note}
\newcommand{\norm}[1]{\left\lVert#1\right\rVert}
\newcommand{\expect}[1]{\mathbb{E}\left[{#1}\right]}
\newcommand{\prob}[1]{\mathbb{P}\left[{#1}\right]}
\newcommand{\given}{\; \big\vert \;}

\newcommand{\argmax}{\mathop{\mathrm{argmax}}}

\newcommand{\argmin}{\mathop{\mathrm{argmin}}}

\newcommand{\clip}{\texttt{Clip}}
\newcommand{\mC}{\mathcal{C}}

\newcommand{\mE}{\mathcal{E}}
\newcommand{\mP}{\mathcal{P}}
\newcommand{\reals}{\mathbb{R}}
\newcommand{\mS}{\mathcal{S}}
\newcommand{\mX}{\mathcal{X}}
\newcommand{\mY}{\mathcal{Y}}

\pdfoutput=1
\title{Explicit Best Arm Identification in Linear Bandits Using No-Regret Learners}

\author{\thanks{Under review. Please do not distribute.}
 Mohammadi Zaki\\
 Electrical Communication Engineering,\\
 Indian Institute of Science,\\
 Bangalore 560012. \\
 \texttt{mohammadi@iisc.ac.in},
\And
Avi Mohan \\
 Faculty of Electrical Engineering,\\
 Technion, Israel Institute of Technology,\\
 Haifa 3200003.\\
\texttt{avinashmohan@campus.technion.ac.il} 
\And
Aditya Gopalan\\
 Electrical Communication Engineering,\\
    Indian Institute of Science,\\
    Bangalore 560012. \\
 \texttt{aditya@iisc.ac.in} 
}
\begin{document}
\maketitle
\allowdisplaybreaks
%\begin{center}
%%-- Lecture I --\\
%{\large {\uppercase{\bf Best arm identification in linear bandits}}}
%
%	\textsc{}\\
%	\medskip\textit{}\\
%	\textsc{E-mail:} \texttt{}\\
%\end{center}
\begin{center}
\textbf{\textit{Abstract.}}	\\
\end{center}

We study the problem of best arm identification in linearly parameterised multi-armed bandits. Given a set of feature vectors $\mathcal{X}\subset\mathbb{R}^d,$ a confidence parameter $\delta$ and an unknown vector $\theta^*,$ the goal is to identify $\argmax_{x\in\mathcal{X}}x^T\theta^*$, with probability at least $1-\delta,$ using noisy measurements of the form $x^T\theta^*.$ For this fixed confidence ($\delta$-PAC) setting, we propose an explicitly implementable and provably order-optimal sample-complexity algorithm to solve this problem. Previous approaches rely on access to minimax optimization oracles. 
% provably optimal \textit{and} computationally efficient technique to solve this problem.\\
The algorithm, which we call the {\em Phased Elimination Linear Exploration Game} (PELEG), maintains a high-probability confidence ellipsoid containing $\theta^*$  in each round and uses it to eliminate suboptimal arms in phases. PELEG achieves fast shrinkage of this confidence ellipsoid along the most confusing (i.e., close to, but not optimal) directions by interpreting the problem as a two player zero-sum game, and sequentially converging to its saddle point using low-regret learners to compute players' strategies in each round. %
We analyze the sample complexity of PELEG and show that it matches, up to order, an instance-dependent lower bound on sample complexity in the linear bandit setting. %
% Phased-LinPEG is, thus, the first algorithm to achieve both order-optimal sample complexity and explicit implementability for this setting. 
We also provide numerical results for the proposed algorithm consistent with its theoretical guarantees. 

%\textit{Keywords: } 
%Keyword 1,
%Keyword 2,
%Keyword 3.

%%%%%%%%%%%%%%%%%%

%\begin{center}
%\tableofcontents
%\end{center}
%\input{Introduction}
\section{Introduction}

Function optimization over structured domains is a basic sequential decision making problem. A well-known formulation of this problem is Probably Approximately Correct (PAC) best arm identification in multi-armed bandits \citep{evendar}, in which a learner is given a set of arms with unknown (and unrelated) means. The learner must sequentially test arms and output, as soon as possible with high confidence, a near-optimal arm (where optimality is defined in terms of the largest mean).

% arguably the most basic version is best arm ID in bandits, where a learner has to rapidly figure out the best option among a discrete set of options by adaptively and repeatedly testing them; much is known about this problem; 

Often, the arms (decisions) and their associated rewards, possess structural relationships, allowing for more efficient learning of the rewards and transfer of learnt information, e.g., two `close enough' arms may have similar mean rewards.  One of the best-known examples of structured decision spaces is the linear bandit, whose arms are vectors (points) in $\mathbb{R}^d$. The reward or function value of an arm is an unknown linear function of its vector representation, and the goal is to find an arm with maximum reward in the shortest possible time by measuring arms' rewards sequentially with noise. This framework models an array of structured online linear optimization problems including adaptive routing \citep{awerbuch2008online}, smooth function optimization over graphs \citep{valko2014spectral}, subset selection \citep{kuroki2019polynomial} and, in the nonparametric setting, black box optimization in smooth function spaces \citep{srinivas2009gaussian}, among others.

Although no-regret online learning for linear bandits is a well-understood problem (see \citep{LatSze2020book} and references therein), the PAC-sample complexity of best arm identification in this model has not received significant attention until recently \citep{soare}. The state of the art here is the work of \citet{jamieson-etal19transductive-linear-bandits}, who give an algorithm with optimal (instance-dependent) PAC sample complexity. However, a closer look indicates that the algorithm assumes repeated oracle access to a minimax optimization problem\footnote{This is, in fact, a plug-in version of a minimax optimization problem representing an information-theoretic sample complexity lower bound for the problem.}; it is not clear, from a performance standpoint, in what manner (and to what accuracy) this optimization problem should be {\em practically} solved\footnote{For its experiments, the paper implements a (approximate) minimax oracle using the Frank-Wolfe algorithm and a heuristic stopping rule, but this is not rigorously justifiable for nonsmooth optimization, see Sec. \ref{sec:Algorithm}.} to enjoy the claimed sample complexity. Hence, the question of how to design an explicit algorithm with optimal PAC sample complexity for best arm identification in linear bandits has remained open. 

In this paper, we resolve this question affirmatively by giving an explicit linear bandit best-arm identification algorithm with instance-optimal PAC sample complexity and, more importantly, a clearly quantified computational effort. % \ag{Can/should we say 'Polynomial in ...'?} 
We achieve this goal using new techniques: the main ingredient in the proposed algorithm is a game-theoretic interpretation of the minimax optimization problem that is at the heart of the instance-based sample complexity lower bound. This in turn yields an adaptive, sample-based approach using carefully constructed confidence sets for the unknown parameter $\theta^*$. The adaptive sampling strategy is driven by the interaction of 2 no-regret online learning subroutines that attempt to solve the minimax problem approximately, obviating the worry of i) solving the optimal minimax allocation to a suitable precision and ii) making an integer sampling allocation from it by rounding, which occur in the approach of Fiez et al  \cite{jamieson-etal19transductive-linear-bandits}. We note that the seeds of this game-theoretic approach were laid by the recent work of \citet{degenne-etal19non-asymptotic-exploration-solving-games} for the simple (i.e., unstructured) multiarmed bandit problem. However, our work demonstrates a novel extension of their methodology to solve best-arm learning in structured multi-armed bandits for the first time to the best of our knowledge.

\subsection{Related Work}
The PAC best arm identification problem for linearly parameterised bandits is  first studied in \citep{soare}, in which an adaptive algorithm is given with a sample complexity guarantee involving a hardness term ($M^*$) which in general renders the sample complexity suboptimal. Tao et al \citep{Tao_et-al-ICML-2018} take the path of constructing new estimators instead of ordinary least squares, using which they give an algorithm achieving the familiar sum-of-inverse-gaps sample complexity known for standard bandits; this is, however, not optimal for general linear bandits. The LinGapE algorithm \citep{XuAISTATS}
 is an attempt at solving best arm identification with a {\em fully adaptive} strategy, but its sample complexity in general is not instance-optimal and can additionally scale with the total number of arms, in addition to the extra dimension-dependence known to be incurred by self-normalized inequality-based confidence set constructions \citep{AbbPalCsa11:linbandits}. Zaki et al \citep{zaki-etal19towards-optimal-efficient-bai-GLUCB} design a fully adaptive algorithm based on the Lower-Upper Confidence Bound (LUCB) principle \citep{Kalyanakrishnan2012PACSS} with limited guarantees for 2 or 3 armed settings. Fiez et al \cite{jamieson-etal19transductive-linear-bandits} give a phased elimination algorithm achieving the ideal information-theoretic sample complexity but with minimax oracle access and an additional rounding operation; we detail an explicit arm-playing strategy that eliminates both these steps, in the same high-level template.  In a separate vein, game-theoretic techniques to solve minimax problems have been in existence for over a couple of decades \citep{freund1999adaptive}; only recently have they been combined with optimism to give a powerful framework to solve adaptive hypothesis testing problems \citep{degenne-etal19non-asymptotic-exploration-solving-games}. 
 
 Table~\ref{table:sampleComplexityComparison} compares the sample complexities of various best arm identification algorithms in the literature.

% ================= \\
% Note: \textit{In the prior work subsection we need to draw distinctions between our own prior work on GLUCB and this one as well.} \\
% Include small paragraph on no-regret learners to solve games - \cite{degenne-etal19non-asymptotic-exploration-solving-games} has some good references on Pg.~3.\\

% \begin{color}{red}
% \subsection{Our Contributions and Organization}
% This could be a good roadmap to follow:
% \begin{enumerate}
%     \item Problem Statement, notation and a remark about lower bounds from Jamieson's paper: \avim{see first paragraph on Page 2 of Degenne's paper: very concisely describes the problem.}
%     \begin{itemize}
%         \item Table comparing sample complexity and some remarks about computational complexity, since we're pushing for {\em efficient} BAI.
%     \end{itemize}
%     \item A section on how we've used the techniques in \cite{degenne-etal19non-asymptotic-exploration-solving-games} to develop Phased-LinPEG and state the algorithm. End section by explaining how the algorithm proceeds in words!!
%     \item Algorithmic ingredients -- Algorithm pseudocode explanation and proof sketch with appropriate figures
%     \item Experiments
%     \item Conclude and remark on future work.
% \end{enumerate}
% \end{color}

% % % % % % % % % % % % % % % % % % % % % % % % % % % % % % % % % % % % %
\section{Problem Statement and Notation}\label{Problem}

We study the problem of best arm identification in linear bandits with the arm set $\mathcal{X}\equiv \{x_1,x_2,\ldots,x_K \}$, where each arm\footnote{In general, the number of arms can be much larger than the ambient dimension, i.e., $d\ll K.$} $x_a$ is a vector in $\mathbb{R}^d$. We will interchangeably use $\mathcal{X}$ and the set $[K]\equiv\{1,2,\ldots,K\}$, whenever the context is clear. In every round $t=1,2,\ldots$ the agent chooses an arm $x_t \in  \mathcal{X} $, and receives a reward $y(x_t)={\theta^*}^Tx_t + \eta_t$, where $\theta^*$ is assumed to be a fixed but unknown vector, and $\eta_t$ is zero-mean noise assumed to be conditionally $1-$ subgaussian, i.e., $\forall \gamma\in \mathbb{R}, \expect{e^{\gamma \eta_t}|x_{1},x_{2},\ldots,x_{{t-1}}, \eta_1,\eta_2,\ldots,\eta_{t-1}} \leq \exp\Big(\frac{\gamma^2}{2} \Big)$. 
We denote by $\nu^k_{\theta^*}$ the distribution of the reward obtained by pulling arm $k\in[K],$ i.e., $\forall t\geq1, y(x_t)\sim\nu^k_{\theta^*},$ whenever $x_t=x_k.$ Given two probability distributions $\mu,\nu$ over $\reals$, $KL(\mu,\nu)$ denotes the KL Divergence of $\mu$ and $\nu$ (assuming $\mu\ll\nu$). Given $\theta\in\reals^d$, let $a^* \equiv a^*(\theta)=\argmax\limits_{a\in [K]} {\theta}^Tx_a$, where we assume that $\theta$ is such that the argmax is unique.

A learning algorithm for the best arm identification problem comprises the following rules: (1) a {\em sampling rule}, which determines based on the past play of arms and observations, which arm to pull next, (2) a {\em stopping rule}, which controls the end of sampling phase and is a function of the past observations and reward, and (3) a {\em recommendation rule}, which, when the algorithm stops, offers a guess for the best arm. The goal of a learning algorithm is: Given an error probability $\delta>0$,  identify (guess) $ a^* $ with probability $\geq 1-\delta$ by pulling as few (in an expected sense) arms as possible. Any algorithm that (1) stops with probability 1 and (2) returns $a^*$ upon stopping with probability at least $1-\delta$ is said to be  \emph{$\delta$-Probably Approximately Correct ($\delta$-PAC).} For clarity of exposition, we distinguish the above {\em linear bandit} setting from what we term the {\em unstructured bandit} setting, wherein $K=d,$ and $x_i=\hat{e}_i,~\forall k\in[K]$ the canonical basis vectors (the former setting generalizes the latter). The (expected) number of samples $\tau\in\mathbb{N}$ consumed by an algorithm in determining the optimal arm in any bandit setting (not necessarily the {\em linear} setting) is called its {\em sample complexity.}

  \par In the rest of the paper, we will assume that $\norm{x_k}_2\leq 1, \forall x_k\in \mathcal{X}$. Given a positive definite matrix $A$, we denote by $\norm{x}_{A}:=\sqrt{x^TAx}$, the matrix norm induced by $A$. For any $i\in [K], i\neq a^*$, we define $\Delta_{i}:={\theta^*}^T(x_{a^*}-x_{i})$ to be the gap between the largest expected reward and the expected reward of (suboptimal) arm $x_i$. Let $\Delta_{\min}:=\min\limits_{\substack{i\in[K]}} \Delta_i$. We denote $B(z,r)$ as the \textit{closed} ball with center $z$ and radius $r$. For any measurable space $\left(\Omega,\mathcal{F}\right),$ we define $\mP(\Omega)$ to be the set of all probability measures on $\Omega$. $\tilde{\mathcal{O}}$ is big-Oh notation that suppresses logarithmic dependence on problem parameters. For the benefit of the reader, we provide a glossary of commonly used symbols in Sec.~\ref{sec:glossaryOfSymbols} in the Appendix.
   
    %------------------------------------- Sample Complexity Comparison Table ---------------------------------------
    
    \begin{table}[tbh]

  %  \label{Synthetic dataset 1 }
    %\centering
    \resizebox{\textwidth}{!}{%
    
    \begin{tabular}{c|c|c}
      \toprule
      %\multicolumn{2}{c}{Part}                   \\
      %\cmidrule(r){1-2}
      {\bf Algorithm} & {\bf Sample Complexity} & {\bf Remarks}\\
      \hline
      {$\mathcal{X} \mathcal{Y}$-static} \citep{soare} & $\mathcal{O}\Big(\frac{d}{\Delta_{\min}}(\ln\frac{1}{\delta} + \ln K + \ln \frac{1}{\Delta_{min}}) +d^2\Big)$ & \begin{tabular}{c} Static allocation, worst-case optimal\\ Dependence on $d$ cannot be removed\end{tabular}\\
      {LinGapE}\footnote{Here $H_0$ is a complicated term defined in terms of a solution to an offline optimization problem in \cite{XuAISTATS}.} \citep{XuAISTATS} & $\mathcal{O}\Big(dH_0\log\Big(dH_0\log\frac{1}{\delta}\Big) \Big)$ & Fully adaptive, sub-optimal in general.\\
    %   {\color{blue}$\mathcal{Y}$-ElimTil-$p$} \citep{Tao_et-al-ICML-2018} & $\mathcal{O}\Big(\frac{d}{\Delta_{min}}(\ln\frac{1}{\delta} + \ln K + \ln\ln \frac{1}{\Delta_{min}})\Big)$\\
      {ALBA} \citep{Tao_et-al-ICML-2018} & $\mathcal{O}\Big( \sum_{i=1}^d \frac{1}{\Delta^2_{(i)}} \ln\left(\frac{K}{\delta}+ \ln \frac{1}{\Delta_{min}}\right) \Big)$ & Fully adaptive, sub-optimal in general (see \cite{jamieson-etal19transductive-linear-bandits})\\
     {RAGE} \citep{jamieson-etal19transductive-linear-bandits} & $\mathcal{O}\left(\frac{1}{D_{\theta^*}}\log{1/\Delta_{\min}}\log{\left(\frac{K^2\log^2{1/\Delta_{min}}}{\delta}\right)} \right)$ & \begin{tabular}{c}Instance-optimal, but  \\ Minimax oracle required \end{tabular} \\
     {PELEG (this paper)} & $\mathcal{O}\left(\frac{\log_2\left(1/\Delta_{min}\right)}{D_{\theta^*}}\left[\frac{\log^2\left(\left(\log_2\left(1/\Delta_{min}\right)\right)^2K^2/\delta\right)}{C^2} \right] \right)$ &  \begin{tabular}{c}{Instance-optimal} (upto a factor of $ C^2 $),\\ \bf{Explicitly implementable} \\\end{tabular} \\
     %Instance-dependent lower bound (upto log factors)  \\
      \bottomrule
    \end{tabular}}
  
  \caption{Comparison of Sample complexities achieved by various algorithms for the Linear Multi-armed Bandit problem in the literature. Note that $K$ is the number of arms, $d$ is the ambient dimension, $\delta$ is the PAC guarantee parameter and $\Delta_{min}$ is the minimum reward gap. $H_0$ is a complicated term defined in terms of a solution to an offline optimization problem in \cite{XuAISTATS}.
    }
    \label{table:sampleComplexityComparison}
  \end{table}
  
Note: $C:=\lambda_{\min}\left(\sum_{x\in \mX}xx^T\right)$  is a term depending only on the geometry of the arm set.  
   %---------------------------------------- Sample Complexity Comparison Table ----------------------------------------------
  %\vspace{-0.25cm}
 \section{Overview of Algorithmic Techniques}\label{sec:algorithmicIngredients}
 %\vspace*{-0.25cm}
 %\subsection{Lower Bounds on Exploration}\label{sec:lowerBoundsOnExploration}
 In this section we describe the main ingredients in our algorithm design and how they build upon ideas introduced in recent work \cite{degenne-etal19non-asymptotic-exploration-solving-games, jamieson-etal19transductive-linear-bandits} (the explicit algorithm appears in Sec. \ref{sec:Algorithm}).
 
{\bf The phased elimination approach: Fiez et al \cite{jamieson-etal19transductive-linear-bandits}.} 
  We first note that a lower bound on the sample complexity of any $\delta$- PAC algorithm for the canonical (i.e., unstructured) bandit setting \cite{garivier-kaufmann16optimal-best-arm-fixed-confidence} was generalized by Fiez et al  \cite{jamieson-etal19transductive-linear-bandits} to the linear bandit setting, assuming $\{\eta_t\}_{t\geq1}$ to be standard normal random variables. This result states that any $\delta$-PAC algorithm in the linear setting must satisfy $\mathbb{E}_{\theta^*}[\tau] \geq \left(\log{1/2.4\delta}\right) \frac{1}{T_{\theta^*}} \geq \left(\log{1/2.4\delta}\right) \frac{1}{D_{\theta^*}}$, %
  %\begin{align}
%\begin{eqnarray}
%      \mathbb{E}_{\theta^*}[\tau] &\geq& \left(\log{1/2.4\delta}\right) \frac{1}{T_{\theta^*}} \nonumber\\
%%      &\geq& \left(\log{1/2.4\delta}\right) \min_{w\in\mP(\mX)} \max_{x\in\mathcal{X}\setminus\{x^*\}}\frac{\parallel x^*-x \parallel^2_{\left(\sum_{x\in\mathcal{X}}w_xxx^T\right)^{-1}}}{\left(\left(x^*-x\right)\theta^*\right)^2}\label{eqn:minMaxOracleForJamieson}\\
%      &=& \left(\log{1/2.4\delta}\right) \frac{1}{D_{\theta^*}},
%      \label{eqn:linearBanditLowerBound}
%\end{eqnarray}
  %\end{align}
  where $T_{\theta^*}:= \max_{w\in\mP(\mX)} \min_{\theta:a^*(\theta)\neq a^*(\theta^*)} \sum_{k\in[K]}w_kKL\left(\nu^k_{\theta},\nu^k_{\theta^*}\right)$ and $D_{\theta^*}:=\max\limits_{w\in\mP(\mX)} \min\limits_{x\in \mX, x\neq x^*}\frac{\left({\theta^*}^T\left(x^*-x\right)\right)^2}{\norm{x^*-x}^2_{\left(\sum_{x\in\mathcal{X}}w_xxx^T\right)^{-1}}}$, where $x^*=x_{a^*}$. The bound suggests a natural $\delta$-PAC strategy, namely, to sample arms according to the distribution 
  \begin{equation}
      w^* = \argmin_{w\in\mP(\mX)} \max_{x\in\mathcal{X}\setminus\{x^*\}}\frac{\parallel x^*-x \parallel^2_{\left(\sum_{x\in\mathcal{X}}w_xxx^T\right)^{-1}}}{\left(\left(x^*-x\right)\theta^*\right)^2}. \label{eq:wStar}
  \end{equation}
  In fact, as \cite[Sec.~2.2]{jamieson-etal19transductive-linear-bandits}  explains, using the Ordinary Least Squares (OLS) estimator $\hat{\theta}$ for $\theta^*$ and sampling arm~$x\in\mX$ exactly $2\lfloor w^*N \rfloor$ times with $N=\mathcal{O}\left(\frac{\log{K/\delta}}{D_{\theta^*}}\right)$ ensures $(x-x^*)^T\hat{\theta}>0,~\forall x\neq x^*$ with probability $\geq1-\delta.$ Unfortunately, this sampling distribution cannot directly be implemented since $x^*$ is unknown. 
  
  Fiez et al circumvent this difficulty by designing a nontrivial strategy (RAGE) that attempts to mimic the optimal allocation $w^*$ in phases. Specifically, in phase $m$, it tries to eliminate arms that are about $2^{-m}$-suboptimal (in their gaps), by solving (\ref{eq:wStar}) with a plugin estimate of $\theta^*$. The resulting fractional allocation, passed through a separate discrete rounding procedure, gives an integer pull count distribution which ensures that all surviving arms' mean differences are estimated with high precision and confidence. 
  
  Though direct and appealing, this phased elimination strategy is based crucially on solving minimax problems of the form (\ref{eq:wStar}). Though the inner (max) function is convex as a function of $w$ on the probability simplex (see e.g., Lemma 1 in \citep{zaki-etal19towards-optimal-efficient-bai-GLUCB}), it is {\em non-smooth}, and it is not made explicit  how, and to what extent, it must be solved in \cite{jamieson-etal19transductive-linear-bandits}. Fortunately, we are able to circumvent this obstacle by using ideas from games between no-regret online learners with optimism, as introduced by the work of Degenne et al \citep{degenne-etal19non-asymptotic-exploration-solving-games} for unstructured bandits.

%   \avim{Discuss how \cite{jamieson-etal19transductive-linear-bandits} leverage this obsrn. and disadvantages of that approach. Segue into Degenne etal.}
  
%   \ag{Also point out the lines in Fiez et al's alg that require a costly minmax oracle, and say that we overcome this using techniques from Degenne, before moving to the next subsec.}
  
 {\bf From Pure-exploration Games to $\delta$-PAC Algorithms: Degenne et al \cite{degenne-etal19non-asymptotic-exploration-solving-games}.}
 We briefly explain some of the insights in \cite{degenne-etal19non-asymptotic-exploration-solving-games} that we leverage to design an explicit linear bandit-$\delta$-PAC algorithm with \emph{low computational complexity.} For a fixed weight parameter $\theta^*\in\reals^d,$ consider the two-player, zero-sum \emph{Pure-exploration Game} in which the $MAX$ player (or column player) plays an arm $k\in[K]$ while the $MIN$ (or row) player chooses an alternative bandit model $\theta\in\reals^d$ such that $a^*(\theta)\neq a^*.$ $MAX$ then receives a payoff of $\sum_{k\in[K]}KL\left(\nu^k_{\theta^*},\nu^k_{\theta}\right)$ from $MIN.$ For a given $w\in\mP(\mX),$ define $T_{\theta^*}(w) = \min_{\theta:a^*(\theta)\neq a^*}\sum_{x\in\mX}w_xxx^T,$ and $w^*(\theta^*)$ the mixed strategy that attains $T_{\theta^*}.$ 
 With $MAX$ moving first and playing a mixed strategy $w\in\mP(\mX),$ the value of the game becomes $T_{\theta^*}$. In the unstructured bandit setting, to match the sample complexity lower bound, any algorithm must essentially sample arm $k\in[K]$ at rate $\frac{N^K_t}{t}\rightarrow w^*_k(\theta^*),$ where $N^k_t$ is the number of times Arm~$k$ has been sampled up to time $t$ \citep{Kaufman16}. This helps explain why any $\delta$-PAC algorithm implicitly needs to solve the Pure Exploration Game $T_{\theta^*}.$ 
 
 We crucially employ no-regret online learners to solve the Pure Exploration Game for linear bandits. More precisely, no-regret learning with the well-known Exponential Weights rule/Negative-entropy mirror descent algorithm \citep{shalev2011online} on one hand, and a best-response convex programming subroutine on the other, provides a {\em direct} sampling strategy that obviates the need for separate allocation optimization and rounding for sampling as in \citep{jamieson-etal19transductive-linear-bandits}. One crucial advantage of our approach (inspired by \cite{degenne-etal19non-asymptotic-exploration-solving-games}) is that we only use a best response oracle to solve for $T_{\theta^*}(w)$, which gives us a computational edge over \cite{jamieson-etal19transductive-linear-bandits} who employ the  computationally more costly max-min oracle to solve $T_{\theta^*}(w),$ or, its linear bandit equivalent, $D_{\theta^*}.$
\section{Algorithm and Sample Complexity Bound}\label{sec:Algorithm}
%\ag{Alg name: PELEG}
%We now formally state the main algorithm of this paper. 

% \begin{enumerate}
% \item a {\color{blue}sampling rule}: which determines based on the past play of arms and observations, which arm to pull next,
% \item a {\color{blue}stopping rule}: which controls the end of sampling phase and is a function of the past observations and rewards only, and
% \item a {\color{blue}recommendation rule}: which when the algorithm stops, reports the index of the arm concluded by the algorithm as guess for the best arm.
% \end{enumerate}
Our algorithm, that we call \enquote{Phased Elimination Linear Exploration Game} (PELEG), is presented in detail as  Algorithm~\ref{alg:PEPEG-mixed}. PELEG proceeds in phases with each phase consisting of multiple rounds, maintaining a set of \emph{active} arms $\mX_m$ for testing during Phase~$m$. An OLS estimate $\hat{\theta}_m$ of $\theta^*$ is used to estimate the mean reward of active arms and, at the end of phase $m$, every active arm with a plausible reward more than $\approx 2^{-m}$ below that of some arm in $\mX_m$ is eliminated. %Clearly, therefore, PELEG terminates within $\lceil\log_2(1/\Delta_{\min})\rceil$ phases. 
Suppose $\mS_m:=\left\{ x\in\mX\setminus\{x^*\}: {\theta^*}^T\left(x^*-x\right)<\frac{1}{2^{m}}\right\}$. If we can ensure that $\mX_m\subset\mS_m$ in every Phase~$m\geq1,$ then PELEG will terminate within $\lceil\log_2(1/\Delta_{\min})\rceil$ phases, where $\Delta_{\min} = \min_{x\neq x^*}{\theta^*}^T\left(x^*-x\right).$ This statement is proved in Corollary~\ref{corollary:finite number of phases} in the Supplementary Material.

If we knew $\theta^*,$ then we could sample arms according to the optimal distribution $w^*$ in \eqref{eq:wStar}. However, since all we now have at our disposal is the knowledge that $\Delta_i\leq2^{-m},~\forall x_i\in\mX_m,$ we can instead construct a sampling distribution $w^*_m$ by solving the surrogate  
$w^*_m = \argmin_{w\in\mP(\mX)}\max_{x,x'\in\mX_m:x\neq x'}\parallel x-x' \parallel^2_{\left(\sum_{x\in\mX}w_xxx^T\right)^{-1}}$,
    and sampling each arm in $\mX_m$ sufficiently often to produce a small enough confidence set. This is precisely what RAGE \citep{jamieson-etal19transductive-linear-bandits} does. However solving this optimization is, as mentioned in Sec.~\ref{sec:algorithmicIngredients}, computationally expensive and RAGE repeatedly accesses a minmax oracle to do this. Note that in simulating this algorithm, the authors implement an approximate oracle using the Frank-Wolfe method to solve the outer optimization over $w$ \cite[Sec.~F]{jamieson-etal19transductive-linear-bandits}. The $\max$ operation, however, renders the optimization objective non-smooth, and it is well-known that the Frank-Wolfe iteration can fail with even simple non-differentiable objectives (see e.g., \cite{cheung-yuying18solving-separable-nonsmooth-franke-wolf}). We, therefore, deviate from RAGE at this point by employing three novel techniques, the first two motivated by ideas in \cite{degenne-etal19non-asymptotic-exploration-solving-games}.
\begin{itemize}
    \item We formulate the above minimax problem as a two player, zero-sum game. We solve the game sequentially, converging to its Nash equilibrium by invoking the use of the EXP-WTS algorithm \cite{cesa-bianchi06prediction-learning-games}. Specifically, in each round $t$ in a phase, PELEG supplies EXP-WTS with an appropriate loss function $l^{MAX}_{t-1}$ and receives the requisite sampling distribution $w_t$ (lines \ref{PELEGStep:ExpWtsUsage1} \& \ref{PELEGStep:ExpWtsUsage2} of the algorithm). This $w_t$ is then fed to the second no-regret learner -- a best response subroutine -- that finds the `most confusing' plausible model $\lambda$ to focus next on (line \ref{PELEGStep:minPlayerStrategyLambdaT}). This is a minimization of a quadratic function over a union of finitely many convex sets (halfspaces intersecting a ball) which can be transparently implemented in polynomial time. 
   
     % We provide more details regarding this in Sec.~\ref{sec:sketchSampleCmplxtyAnalysis}.
    \item Once the sampling distribution is found, there still remains the problem of actually sampling according to it. Given a distribution $w\in\mP(\mX_m),$ approximating it by sampling $x\in\mX$ $\lfloor Nw_x\rfloor$ or $\lceil Nw_x\rceil$ times can lead to too few (resp. many) samples. Other naive sampling strategies are, for the same reason, unusable. While \cite{jamieson-etal19transductive-linear-bandits} invokes a specialized \emph{rounding algorithm} for this purpose, we opt for a more efficient \emph{tracking} procedure (line \ref{PELEGstep:tracking}): In each Round~$t$ of Phase~$m$, we sample Arm~$k_t:=\argmin\limits_{k\in [K]}n_{t-1}^k/\sum\limits_{s=1}^{t}w_s^k$, where $n^k_t$ is the number of times Arm~$k$ has been sampled up to time $t.$ In Lem.~\ref{lemma:tracking}, we show that this procedure is efficient, i.e., $ \sum\limits_{s=1}^{t}w_s^k-\left(K-1\right)\leq n_t^k\leq \sum\limits_{s=1}^{t}w_s^k +1.$
    \item Finally, in each phase~$m$, we need to sample arms often enough to (i) construct confidence intervals of size at most $2^{-(m+1)}$ around $(x-x')^T\theta^*,~\forall x,x'\in\mX_m,$ (ii) ensure that $\mX_m\subset\mS_m$ and (iii) that $x^*\in\mX_m.$ In Sec.~\ref{sec:justificationOfEliminationCriteria}, we prove a Key Lemma (whose argument is discussed in Sec.~\ref{sec:sketchSampleCmplxtyAnalysis}) to show that our novel \emph{Phase Stopping Criterion} ensures this with high probability.
\end{itemize}

It is worth remarking that naively trying to adapt the strategy of Degenne et al \cite{degenne-etal19non-asymptotic-exploration-solving-games} to the linear bandit structure yields a suboptimal (multiplicative $\sqrt{d}$) dependence in the sample complexity, thus we adopt the phased elimination template of Fiez et al \cite{jamieson-etal19transductive-linear-bandits}. We also find, interestingly, that this phased structure eliminates the need to use more complex, self-tuning online learners like AdaHedge \citep{AdaHedge} in favour of the simpler Exponential Weights (Hedge). 
%   
   % % % % % % % % % % % % % % % % % % % % % % % % % % % % % % % % % % % % % % %
   
   \begin{algorithm}[!htb]
      \allowdisplaybreaks
       	\caption{{\bf Phased Elimination Linear Exploration Game (PELEG)}} \label{alg:PEPEG-mixed}
       	\begin{algorithmic}[1]
       	
       	\State \textbf{Input:} $\mX, \delta.$
       	
       	\State \textbf{Init:} $m\leftarrow 1, \mX_m \leftarrow \mX .$
       	
       	\State $C\leftarrow\lambda_{min}\left(\sum\limits_{k=1}^{K}x_kx_k^T\right).$
       	
       	\While {$\left\{\abs{\mX_m} >1 \right\}$}
       	
       	\State $\delta_m \leftarrow \frac{\delta}{m^2}.$
       	
       	\State $D_m \leftarrow 2(\sqrt{2}-1)\sqrt{\frac{{C}}{{\max\limits_{x,x'\in \mX_m, x\neq x'}\norm{x-x'}_2^2\log K}}}$
       	
       	\State $\epsilon_m \leftarrow \min\left\{1, \frac{D_m\sqrt{C}}{\sqrt{8\log\left(K^2/\delta_m\right)}} \right\}\left(\frac{1}{2}\right)^{m+1}. $ \label{PELEGstep:defnOfEpsilonM}
       	
       	\State $\forall x\in \mX_m, \mathcal{C}_m(x):= \left\{\lambda\in \Real^d: \exists x'\in \mX_m, x'\neq x| \lambda^Tx'\geq\lambda^Tx+\epsilon_m \right\}$.

       	\State $ t\leftarrow 1, n_0^k\leftarrow 0, \forall k\in [K] $.
       	
       	\State Play each arm in $\mX$ once and collect rewards $Y_k\sim \nu_k, 1\leq k\leq K$. \Comment{{\color{blue}Burn-in period}}
       	
       	\State $\forall k\in [K],~n_k^t\bigg|_{t=K}=n_k^K\leftarrow 1$, $V_t^m\bigg|_{t=K}=V_K^m\leftarrow\sum\limits_{k=1}^{K}x_kx_k^T$, $ t\leftarrow K $. 
       	
       	%\While $\left\{\max\limits_{y\in \mY\left(\mX \right)}\norm{y}^2_{{V_t}^{-1}} \geq \frac{1}{\left(2^{m+1}\right)^2\log \left(K^2/{\delta_m}\right)} \right\}$
       	
       	\State Initialize $\mathcal{A}^{MAX}_m\equiv EXP-WTS$ with expert set $\{\hat{e}_1,\cdots,\hat{e}_K\}\subset\reals^K$ and loss function $l^{MAX}_{t-1}()$.{\color{blue}\Comment{$MAX$ player: EXP-WTS}}

       	\While {$\left\{\min\limits_{\lambda \in \bigcup\limits_{x\in \mX_m} \mathcal{C}_m(x)\cap B\left(0,D_m\right) } \norm{\lambda}^2_{V_t^m} \leq 8\log \left(K^2/{\delta_m}\right) \right\}$} {\color{blue}\Comment{Phase Stopping Criterion}}
       	\label{PELEGStep:phaseStoppingCriterion}
       	
       	\State $t\leftarrow t+1$.

       	\State Get $w_t$ from $\mathcal{A}^{MAX}_m$ and form the matrix $W_t=\sum\limits_{k=1}^{K}w_t^kx_kx_k^T$. \label{PELEGStep:ExpWtsUsage1}
       	
       	\State $\lambda_t \leftarrow \argmin\limits_{\lambda \in  \cup_{x\in \mX_m}\mathcal{C}_m(x) \cap B\left(0,D_m\right) } \norm{\lambda}^2_{W_t} $.\label{PELEGStep:minPlayerStrategyLambdaT}
       	
%       	\State Compute $\forall k\in [K]: U_t^k=\max\limits_{\theta\in \mE\left(\underline{0}, V_{t-1}^m\right)}\left(\theta^Tx_k-\lambda_t^Tx_k \right)^2$.
		\State For $ k\in [K], U_t^k:= \left(\lambda_t^Tx_k\right)^2.$ {\color{blue}\Comment{$MIN$ player: Best response}}
       	
       	\State Construct loss function $l_t^{MAX}\left(w\right) = -w^TU_t$ . \label{PELEGstep:passLoss2ExpWts} %to $\mathcal{A}^{MAX}_m$ 
       	\label{PELEGStep:ExpWtsUsage2}
       	
       	\State Play arm $k_t:=\argmin\limits_{k\in [K]}\frac{n_{t-1}^k}{\sum\limits_{s=1}^{t}w_s^k}$ {\color{blue}\Comment{Tracking}} \label{PELEGstep:tracking}
       	    	
       	\State $n_{t}^{k_t}\leftarrow n_{t}^{k_t}+1$
       	    	
       	\State Collect sample $Y_t={\theta^*}^Tx_{k_t}+\eta_t $ 
       	        
       	\State $V_{t}^m = V_{t-1}^m+ x_{k_t}{x_{k_t}}^T.$
       	
       	\EndWhile
       	
       	\State $N_{m}\leftarrow t$
       	
       	\State Update: $\hat{\theta}_{m}\leftarrow \left({V^{m}_{N_m}}\right)^{-1}\left(\sum\limits_{s=1}^{N_m}Y_sx_{k_s}\right)$ {\color{blue}\Comment{Least-squares estimate of $\theta^*$}}
       	    
       	\State Update: $\mathcal{X}_{m+1}\leftarrow \mX_m\Big\backslash \left\{ x\in \mathcal{X}_m | \exists x' \in \mX_m: {\hat{\theta}_m}^T\left(x'-x \right)> 2^{-(m+2)}  \right\} $

       	\State $m\leftarrow m+1$

       	\EndWhile
       	
       \State \Return $\mathcal{X}_m$ {\color{blue}\Comment{Output surviving arm}}
       
       	\end{algorithmic}

       \end{algorithm}
  
% \ag{State main thm about sample complexity. Add discussion: $C$ parameter.} 

%\ag{Say that the Phase stopping criterion together with EXP-WTS and Tracking is a substitute for Jamieson's minimax oracle call and rounding (more details in next section).}

%\ag{the $B(0, D_m)$ constraint is required for the sample complexity proof to work. However, in practice it may be removed leading to explicit solutions -- check effect in practice.}

The main theoretical result of this paper is the following performance guarantee.

\begin{theorem}[Sample Complexity of Algorithm \ref{alg:PEPEG-mixed}]\label{theorem:sampleComplexity}
 With probability at least $1-\delta$, PELEG returns the optimal arm after $\tau$ rounds, with
	\begin{equation}
	\begin{split}
	\tau \leq 2048\frac{\log_2\left(1/\Delta_{min}\right)}{D_{\theta^*}}\left[\frac{\left(\log\left(\left(\log_2\left(1/\Delta_{min}\right)\right)^2K^2/\delta\right)\right)^2\log K}{C^2} \right] + \\ 256\frac{\log_2\left(1/\Delta_{min}\right)}{D_{\theta^*}}\log\left(\left(\log_2\left(1/\Delta_{min}\right)\right)^2K^2/\delta\right) = \tilde{\mathcal{O}} \left(\frac{\log^2(K^2/\delta)}{C^2 D_{\theta^*}} \right).
	\end{split}
	\label{eqn:sampleCmpltyOfPELEG}
	\end{equation}
\end{theorem}

In Sec.~\ref{sec:sketchSampleCmplxtyAnalysis}, we sketch the arguments behind the result. The proof in its entirety can be found in Sec.~\ref{appendix:proofOfSampleCmpltyBound} in the Supplementary Material.

\begin{remark}
%\begin{itemize}
 %\item 
 As explained in Sec.~\ref{sec:algorithmicIngredients}, the optimal (oracle) allocation requires $\mathcal{O}\left(\frac{1}{D_{\theta^*}}\log{\frac{K}{\delta}}\right)$ samples. Comparing this with \eqref{eqn:sampleCmpltyOfPELEG}, we see that our algorithm is instance optimal up to logarithmic factors, barring the $\frac{1}{C^2}$ term, so the optimality holds whenever $C=\Omega(1)$. Recall that $C$ is the smallest eigenvalue of $\sum_{x\in\mX}xx^T$. $C=\Omega(1)$ is reasonable to expect given that in most applications, feature vectors (i.e., $x_1,\cdots,x_K$) are chosen to represent the feature space well which translates to a high value of $C$.
%\end{itemize}
\end{remark}
%%%\vspace*{-1cm}
\begin{remark} 
The main computational effort in Algorithm \ref{alg:PEPEG-mixed} is in checking the phase stopping criterion (line \ref{PELEGStep:phaseStoppingCriterion}) and implementing the best-response model learner (line \ref{PELEGStep:minPlayerStrategyLambdaT}), both of which are explicit quadratic programs. Note also that bounding the losses submitted to EXP-WTS to within $B(0,D_m)$ is required only for the regret analysis of EXP-WTS to go through. In practice, as the simulation results show, PELEG works without this and, in fact, permits efficient solution of Step~\ref{PELEGStep:minPlayerStrategyLambdaT} in the algorithm, further reducing computational complexity.
\end{remark}

% We move on to provide some intuition behind the proof of Thm.~\ref{theorem:sampleComplexity}.

% \ag{IMPORTANT: Can't we replace the Burn-in period by ``Play a subset of arms $S \subseteq \mathcal{X}$ such that $\lambda_{\min}(\sum_{k \in S} x_k x_k^T) \geq C$" ? Also mention that if this is done, then for combinatorial arms (i.e., one-hot encoded subsets of $[n]$), we can achieve $\lambda_{\min} \geq n$ by just choosing subsets that cover $[n]$.}

 % % % % % % % % % % % % % % % % % % % % % % % % % % % % % % % % % % %
\section{Sketch of Sample Complexity Analysis}\label{sec:sketchSampleCmplxtyAnalysis}
%\vspace*{-0.2cm}
This section outlines the proof of the $\delta$-PAC sample complexity of Algorithm \ref{alg:PEPEG-mixed} (Theorem \ref{theorem:sampleComplexity}) and describes the main ideas and challenges involved in the analysis.

%As described, the algorithm operates in phases, with the goal of the $m$-th phase being to eliminate, with high probability, arms which are roughly $2^{-m}$-suboptimal w.r.t. other surviving arms and hence whose gaps (to the optimal arm) are at least $2^{-m}$. This can be guaranteed if all available arms (i.e., directions) are sampled sufficiently enough to reduce the uncertainty in estimating mean differences by about $2^{-m}$. To achieve such a sampling pattern, alternating no-regret learners are used in the central {\bf while} loop which is terminated according to a phase stopping criterion (Step~\ref{PELEGStep:phaseStoppingCriterion} of PELEG).

At a high level the proof of Theorem \ref{theorem:sampleComplexity} involves two main parts: (1) a correctness argument for the central {\bf while} loop that eliminates arms, and (2) a bound for its length, which, when added across all phases, gives the overall sample complexity bound. 

% Algorithm \ref{alg:PEPEG-mixed} broadly consists of two modules. We discuss each module briefly here. 
% \begin{enumerate}
%     \item 
% \end{enumerate}

{\bf 1. Ensuring progress (arm elimination) in each phase.} At the heart of the analysis is the following result which guarantees that upon termination of the central while loop, the uncertainty in estimating all differences of means among the surviving (i.e., non-eliminated) arms remains bounded.

\begin{restatable}[Key Lemma]{lemma}{keylemma}
\label{lemma:key lemma on uncertainity after phase m}
After each phase $m\geq 1$,
$ \max\limits_{x,x'\in \mX_m, x\neq x'} \norm{x-x'}^2_{\left({V_{N_m}^m}\right)^{-1}} \leq \frac{\left(\left(\frac{1}{2}\right)^{m+1}\right)^2}{8\log K^2/\delta_m}. $
\end{restatable}

{\bf Proof sketch.} Phase~$m$ ends at time $t$ when the ellipsoid $\mathcal{E}(0,V_t^m,r_m)$, with center $0$ and shape according to the arms played in the phase so far, becomes small enough to avoid intersecting the half spaces $\mathcal{C}_m(x)$, for all surviving arms $x$, within the ball $\cap B(0, D_m)$  (Step~\ref{PELEGStep:phaseStoppingCriterion} of the algorithm) which is required to keep  loss functions bounded for no-regret properties.  

Suppose, for the sake of simplicity, that only two arms $x_i$ and $x_j$ are present when phase $m$ starts. Figure \ref{fig:pfsketch1} depicts a possible situation when the phase ends. $\mathcal{C}_m(x_i) \equiv \mathcal{C}_m(x_i; \epsilon_m)$ and $\mathcal{C}_m(x_j; \epsilon_m)$ with $\epsilon_m \approx 2^{-m}$ are halfspaces, denoted in gray, that intersect the ball $B(0, D_m)$ in the areas colored red. In this situation, the ellipsoid $V_t^m$, shaded in blue, has just broken away from the red regions {\em in the interior of the ball}. Because its extent in the direction $x_i - x_j$ lies within the strip between the two hyperplanes bounding $\mathcal{C}_m(i), \mathcal{C}_m(j)$, it can be shown (see proof of lemma in appendix) that $\norm{x_i-x_j}_{(V_t^m)^{-1}}$ is small enough to not exceed roughly $2^{-m}$. 

The more challenging situation is when the ellipsoid $V_t^m$ breaks away from the red regions by {\em breaching the boundary of the ball} $B(0, D_m)$, as in the green ellipsoid in Figure \ref{fig:pfsketch2}. The {\bf while} loop terminating at this time would not satisfy the objective of controlling $\norm{x_i-x_j}_{(V_t^m)^{-1}}$ to within $2^{-m}$, since  the extent of the ellipsoid in the direction $x_i-x_j$ is larger than the gap between the halfspaces $\mathcal{C}_m(x_i)$ and $\mathcal{C}_m(x_j)$. A key idea we introduce here is to {\em shrink the hyperplane gap} (i.e., $\epsilon_m$) by a factor (precisely $D_m \sqrt{C} (8 \log K^2/\delta_m)^{-1}$) which is represented by the min operation in Step~\ref{PELEGstep:defnOfEpsilonM}. In doing this we bring the halfspaces closer, and then insist that the ellipsoid break away from these {\em new} halfspaces within the ball. This more stringent requirement guarantees that when the loop terminates, the extent of the final ellipsoid (shaded in blue) stays within the original, unshrunk, gap ensuring $\norm{x_i-x_j}_{(V_t^m)^{-1}} \lessapprox 2^{-m}$. 
%\vspace*{-0.9cm}

\begin{figure}[ht]
\begin{subfigure}{.49\textwidth}
  \centering
  % include first image
  \resizebox{0.9\textwidth}{!}{
  \begin{tikzpicture}[>=Stealth]
    \fill[gray!20, shift={+(0,+3)}, rotate=45] (-5.5,0) rectangle (1.3,2);
    \fill[gray!20, shift={+(0,-3)}, rotate=45] (-1.3,0) rectangle (5.5,-2);
    
    \fill[blue!30, rotate=110, thick] (0,0) ellipse (2.2 and 1);
    
    \begin{scope}
      \clip[shift={+(0,3)}, rotate=45] (-5.5,0) rectangle (1,2);
      \fill[red!30] (0,0) circle (3);
    \end{scope}
    \begin{scope}
      \clip[shift={+(0,-3)}, rotate=45] (-2,0) rectangle (5,-2);
      \fill[red!30] (0,0) circle (3);
    \end{scope}
    
    \draw[color=red!60, thin](0,0) circle (3);

    % hyperplanes
    \draw[black, shift={+(-1,+2)}, very thick] (-3.5,-3.5) -- (2.5,2.5);
    
    \draw[black, shift={+(+1,-2)}, very thick] (-2.5,-2.5) -- (3.5,3.5);
    
    %\draw[thick, ->] (0,0) -- (-1,1) node[anchor=west] { $\text{  }x_i - x_j$} ; 
    \filldraw[black] (0,0) circle (2pt) node[anchor=east] {$0$};
    \draw[thin] (-2.5,2.5) -- (3.5,-3.5) node[anchor=west] { $\text{Span}(x_i - x_j)$} ; 
    
    \draw[very thick, ->] (3.5,0.5) -- (3.5+1,0.5-1) node[anchor=north ] {{\small $\quad \quad \{\lambda: \lambda^T(x_i-x_j) \geq \epsilon  \}$}} ;
    
    \draw[very thick, ->] (0.5,3.5) -- (0.5-1,3.5+1) node[anchor=east] {{\small $\{\lambda: \lambda^T(x_i-x_j) \leq -\epsilon  \}$}} ;
    
    \draw[thin, <->] (-4.2,-1.2) -- node[anchor=north east, align=center] {Ellipsoid fits within this\\$\Rightarrow$ $\norm{x_i-x_j}_{(V_t^m)^{-1}} \lessapprox 2^{-m}$ } (-4.2+3,-1.2-3)  ;
    
    \node[color=red, anchor=west] at (2.3,2.3) {{\small $B(0, D_m)$}};
    \node[color=blue, anchor=west] at (0.5,1) {{\small $\mathcal{E}(0, V_t^m, r_m)$}};
        
  \end{tikzpicture}
    }
  \caption{`Easy' case: The blue ellipsoid separates from the halfspaces intersecting the ball (red) by {\em staying within the ball}. In this case its extent along $(x_i-x_j)$ is within the gap between the hyperplanes (parallel black lines).}
  \label{fig:pfsketch1}
\end{subfigure}\hfill
\begin{subfigure}{.49\textwidth}
  \centering
  % include second image
  \resizebox{0.9\textwidth}{!}{
    \begin{tikzpicture}[>=Stealth]
    
    \fill[gray!20, shift={+(0,+3)}, rotate=45] (-5.5,0) rectangle (1.3,2);
    \fill[gray!20, shift={+(0,-3)}, rotate=45] (-1.3,0) rectangle (5.5,-2);
    
    \fill[green!30, rotate=75, thick] (0,0) ellipse (5.5 and 0.8);
    
    \fill[blue!30, rotate=60, thick] (0,0) ellipse (4.2 and 0.6);

    \begin{scope}
      \clip[shift={+(0,3)}, rotate=45] (-5.5,0) rectangle (1,2);
      \fill[red!30] (0,0) circle (3);
    \end{scope}
    \begin{scope}
      \clip[shift={+(0,-3)}, rotate=45] (-2,0) rectangle (5,-2);
      \fill[red!30] (0,0) circle (3);
    \end{scope}

    \draw[color=red!60, thin](0,0) circle (3);
    
    % hyperplanes
    \draw[black, shift={+(-1,+2)}, very thick] (-3.5,-3.5) -- (2.5,2.5);
    \draw[black, shift={+(+1,-2)}, very thick] (-2.5,-2.5) -- (3.5,3.5);
    
    % large gap hyperplanes
    \draw[green, dashed, shift={+(-1,+3)}, thick] (-4,-4) -- (3,3);
    
    \draw[green, dashed, shift={+(+1,-3)}, thick] (-2,-2) -- (5.3,5.3);
    
    % small gap hyperplanes
    \draw[black, dashed, shift={+(-1,0.8)}, thick] (-3,-3) -- (3.5,3.5);
    
    \draw[black, dashed, shift={+(+1,-0.8)}, thick] (-3.5,-3.5) -- (3.3,3.3);

    %\draw[black, dashed, shift={+(-0.5*1,+0.5*2)}, thick] (-3,-3) -- (2.5,2.5);
    %\draw[black, dashed, shift={+(+0.5*1,-0.5*2)}, thick] (-3,-3) -- (3.5,3.5);
    
    \filldraw[black] (0,0) circle (2pt) node[anchor=east] {$0$};
    \draw[thin] (-2.5,2.5) -- (3.5,-3.5) node[anchor=west] { $\text{Span}(x_i - x_j)$} ; 
    
    \draw[very thick, ->] (3.5,0.5) -- (3.5+1,0.5-1) node[anchor=north ] {$\quad \quad \{\lambda: \lambda^T(x_i-x_j) \geq \epsilon  \}$} ;
    
    \draw[very thick, ->] (0.5,3.5) -- (0.5-1,3.5+1) node[anchor=east] {$\{\lambda: \lambda^T(x_i-x_j) \leq -\epsilon  \}$} ;
    
    %\draw[thin, <->] (-4.2,-1.2) -- node[anchor=east] {$\frac{2\epsilon}{|| x_i-x_j||}$} (-4.2+3,-1.2-3)  ;
    
    % right gap to fit ellipsoid
    \draw[thin, <->] (-4.2,-1.2) -- node[anchor=north east, align=center] {Ellipsoid fits within this\\$\Rightarrow$ $\norm{x_i-x_j}_{(V_t^m)^{-1}} \lessapprox 2^{-m}$ } (-4.2+3,-1.2-3)  ;
    
    % ellipsoid gap too large
    \draw[thin, <->, green, shift={+(1.7,5.7)}] node[anchor=west, align=center] {Ellipsoid extent\\too large} (0,0) --  (4,-4)  ;
    
    % ellipsoid gap again small
    \draw[thin, <->, blue, shift={+(2.4,4.2)}] (0,0) --  (1.8,-1.8) node[anchor=south west, align=center] {Ellipsoid extent\\sufficiently small} ;
    
    \end{tikzpicture}
    }

  \caption{{ `Difficult' case: The green ellipsoid separates from the halfspaces intersecting the the ball (red) by {\em  breaching the ball}. Its extent along $(x_i-x_j)$ {\em exceeds the gap} between the hyperplanes (parallel black lines). When forced to separate from a closer pair of halfspaces (dotted black lines), then the ellipsoid's (in blue) extent is within the original gap.}}
  \label{fig:pfsketch2}
\end{subfigure}
\caption{{\footnotesize The phase stopping condition in Algorithm \ref{alg:PEPEG-mixed} ensures $\norm{x_i-x_j}_{(V_t^m)^{-1}} \lessapprox 2^{-m}$ after phase $m$.
% for a pair of surviving actions $x_i, x_j$. 
% The ellipsoid $\mathcal{E}(0, V_t^m, r_m)$ determined by actions played (in blue) is made to separate from the intersection of a suitable pair of halfspaces constrained to the ball $B(0, D_m)$ (in red).
 }}
\label{fig:pfsketch12}
\end{figure}
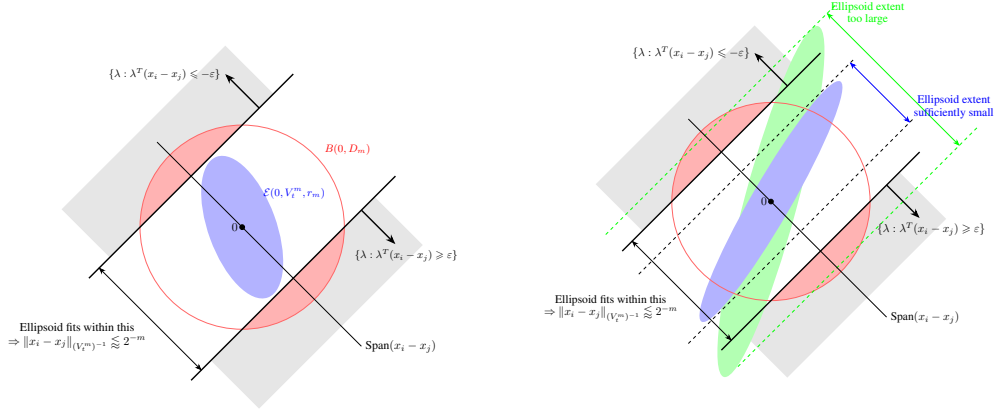

{\bf 2. Bounding the number of arm pulls in a phase.} 
The main bound on the length of the central {\bf while} loop is the following result.
\begin{restatable}[Phase length bound]{lemma}{phaselengthlemma}\label{lemma:phaseLengthBound}
Let $B_m:= \min\limits_{w\in \mathcal{P}_K}\max\limits_{x,x'\in \mX_m, x\neq x'} \norm{x-x'}^2_{W^{-1}}$.
There exists $\delta_0$ such that $\forall \delta<\delta_0$, the length $N_m$ of any phase $m$ is bounded as :
\begin{align*}
N_m \leq \begin{cases}
  2B_m\left(2^{m+1}\right)^2\left[\frac{r_m^4\log K}{(\sqrt{2}-1)^2C^2} \right]+1 &\text{  if } \epsilon_m = \frac{D_m\sqrt{C}}{r_m}\left(\frac{1}{2}\right)^{m+1},\\
 2B_m\left(2^{m+1}\right)^2r_m^2+1&\text{  if } \epsilon_m = \left(\frac{1}{2}\right)^{m+1}.
\end{cases}
\end{align*}

\end{restatable}
%\vspace*{-0.2cm}
To prove this we use the no-regret property of both the best-response $MIN$  and the EXP-WTS $MAX$ learner (the full proof appears in the appendix). A key novelty here is the introduction of the ball $B(0, D_m)$ as a technical device to control the $2$-norm radius of the final stopped ellipsoid $\mathcal{E}(0,V_t^m,r_m)$ (inequality $(i)$ in the proof) when used with the basic tracking rule over arms introduced by Degenne et al \citep{degenne-etal19non-asymptotic-exploration-solving-games}. 

%\vspace*{-0.3cm}

% % % % % % % % % % % % % % % % % % % % % % % % % % % % % % % % % % % %
\section{Experiments}\label{Experiments}
\vspace*{-0.2cm}
We numerically evaluate PELEG, against the algorithms $\mathcal{X}\mathcal{Y}$-static (\cite{soare}), LUCB (\cite{Kalyanakrishnan2012PACSS}), ALBA (\cite{Tao_et-al-ICML-2018}), LinGapE (\cite{XuAISTATS}) and RAGE (\cite{jamieson-etal19transductive-linear-bandits}), for 3 common benchmark settings. The oracle lower bound is also calculated. Note: In our implementation, we ignore the term $B(0,D_m)$ in the phase stopping criterion; this has the advantage of making the criterion check-able in closed form. We simulate independent, $\mathcal{N}(0,1)$ observation noise in each round. All results reported are averaged over 50 trials. We also empirically observe a {$100\%$ success rate} in identifying the best arm, although a confidence value of $\delta=0.1$ is passed in all cases.

{\bf Setting 1: Standard bandit.} The arm set is the standard basis $\left\{e_1,e_2,\ldots,e_5\right\}$ in 5 dimensions. The unknown parameter $\theta^*$ is set to $\left(\Delta,0,\ldots,0 \right)$, where $\Delta>0$, with $\Delta$ swept across $\left\{0.1,0.2,0.3,0.4,0.5\right\}$. As noted in \cite{XuAISTATS}, for $\Delta$ close to $0$, $\mathcal{X}\mathcal{Y}$-static's essentially uniform allocation is optimal, since we have to estimate all directions equally accurately. However, PELEG performs better (Fig. \ref{fig:Experiments}(a)) due to being able to  eliminate suboptimal arms earlier instead of uniformly across all arms. Fig. \ref{fig:Experiments}(b) compares PELEG and RAGE in the smaller window $\Delta\in \left[0.11,0.19\right]$, where PELEG is found to be competitive (and often better than) RAGE. 

{\bf Setting 2: Unit sphere.} The arms set comprises of $100$ vectors sampled uniformly from the surface of the unit sphere $\mathbb{S}^{d-1}$. We pick the two closest arms, say  $u$ and $v$, and then set $\theta^*=u+\gamma(v-u)$ for $\gamma=0.01$, making $u$ the best arm. We simulate all algorithms over dimensions $d=10,20,\ldots,50$. This setting was first introduced in \cite{Tao_et-al-ICML-2018}, and PELEG is uniformly competitive with the other algorithms (Fig. \ref{fig:Experiments}(c)).

{\bf Setting 3: Standard bandit with a confounding arm \cite{soare}.}  We instantiate $d$ canonical basis arms $\{e_1,e_2,\ldots,e_d\}$ and an additional arm $x_{d+1}=(\cos(\omega),\sin(\omega),0,\ldots,0)$, $d=2,\dots,10$, with $\theta^*=e_1$ so that the first arm is the best arm. By setting $0< \omega << 1$, the $d+1$th arm becomes the closest competitor. Here, the performance critically depends on how much an agent focuses on comparing arm $1$ and arm $d+1$. % XY-static doesn't perform well because it treats all directions equally, hence ends up wasting a lot of samples. 
LinGapE performs very well in this setting, and PELEG and RAGE are competitive with it (Fig. \ref{fig:Experiments}(d)). \\
\vspace*{-0.6cm}

\begin{table}[h]
	\hspace*{-0.5cm}
	\parbox{0.21\linewidth}{
		%\centering
		%\hspace*{-0.5cm}
		%\caption{Synthetic dataset 2,  100 trials, $\delta_{min}>0.05$}
		%\label{Synthetic dataset 2}
		%	\begin{figure}
		
		\includegraphics[scale=0.2]{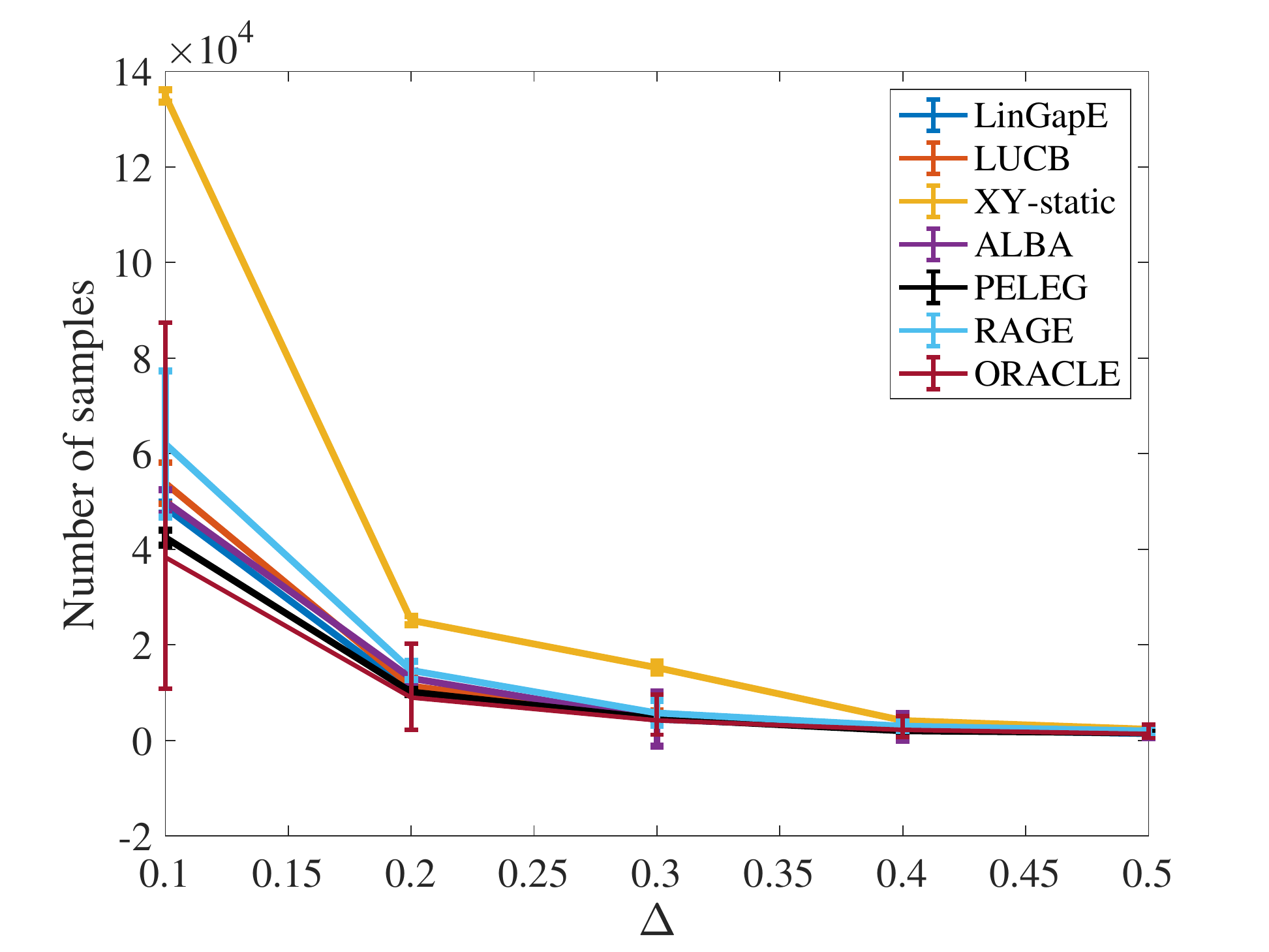}
		%\captionof{figure}{Dataset 1: PELEG performs better than the other algorithms}
		\label{fig:dataset3error-eps-converted-to}
		%\captionof{figure}{(a)}
		%	\label{fig:Dataset aistats}
		%	\end{figure}
	}
	\quad 
	\hspace*{0.2cm}
	\parbox{0.21\linewidth}{
		%\centering
		%\hspace*{-1cm}	%\caption{Synthetic dataset 2,  100 trials, $\delta_{min}>0.05$}
		%\label{Synthetic dataset 2}
		%	\begin{figure}
		\includegraphics[scale=0.2]{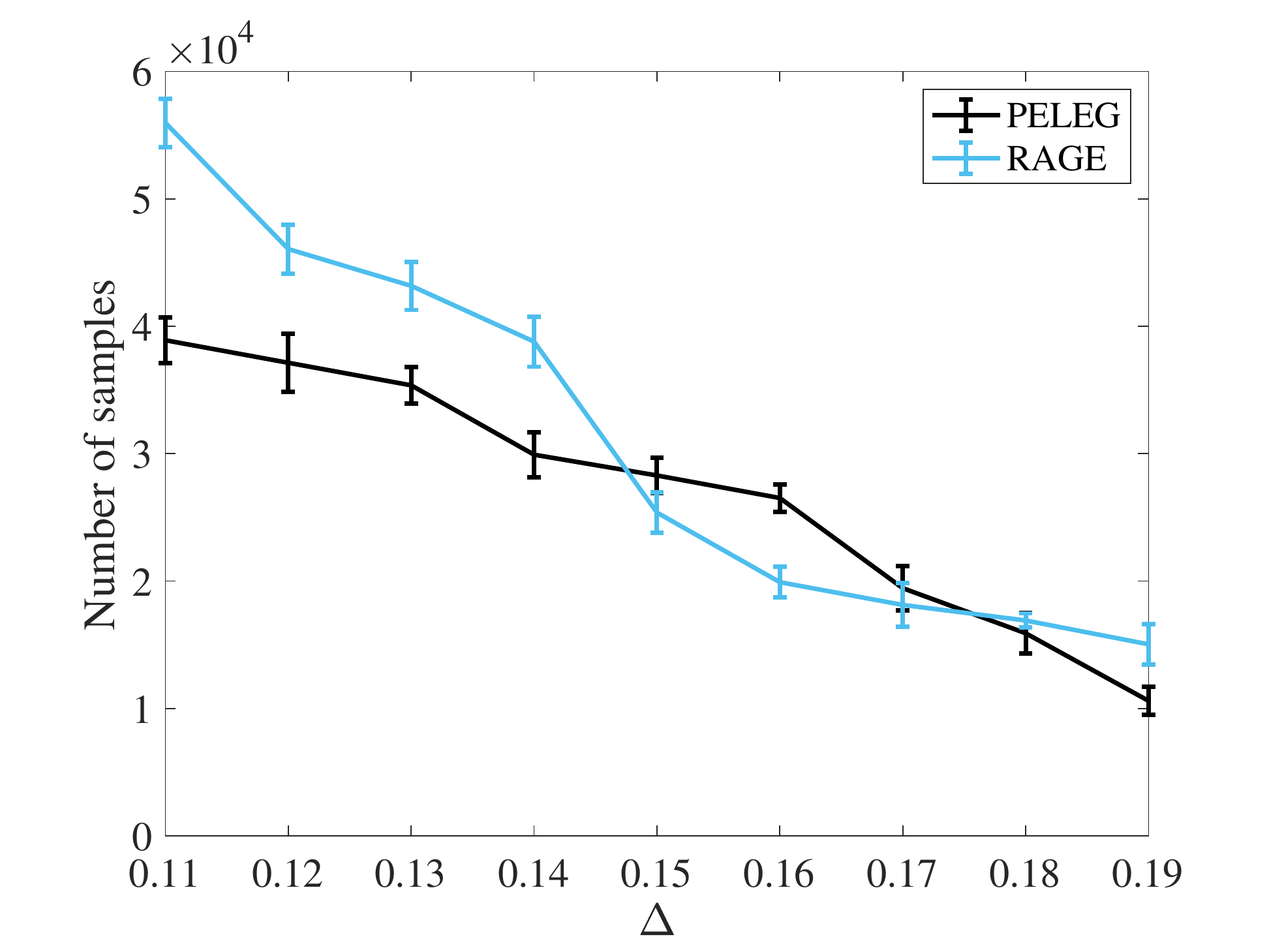}
		%\captionof{figure}{Dataset 1:A closer look in the regime $ \Delta = \{0.11,\ldots,0.19\},$ comparison between PELEG and RAGE .}
		% \captionof{figure}{(a)}
		\label{fig:Dataset aistats extra}
		%	\end{figure}
	}
	\quad
	\hspace*{0.2cm}
	%\end{table}
	%\vspace*{-0.5cm}
	%\begin{table}[h]
	\parbox{0.21\linewidth}{
		%\centering
		%\hspace*{-0.5cm}
		%\caption{Synthetic dataset 2,  100 trials, $\delta_{min}>0.05$}
		%\label{Synthetic dataset 2}
		%	\begin{figure}
		\includegraphics[scale=0.2]{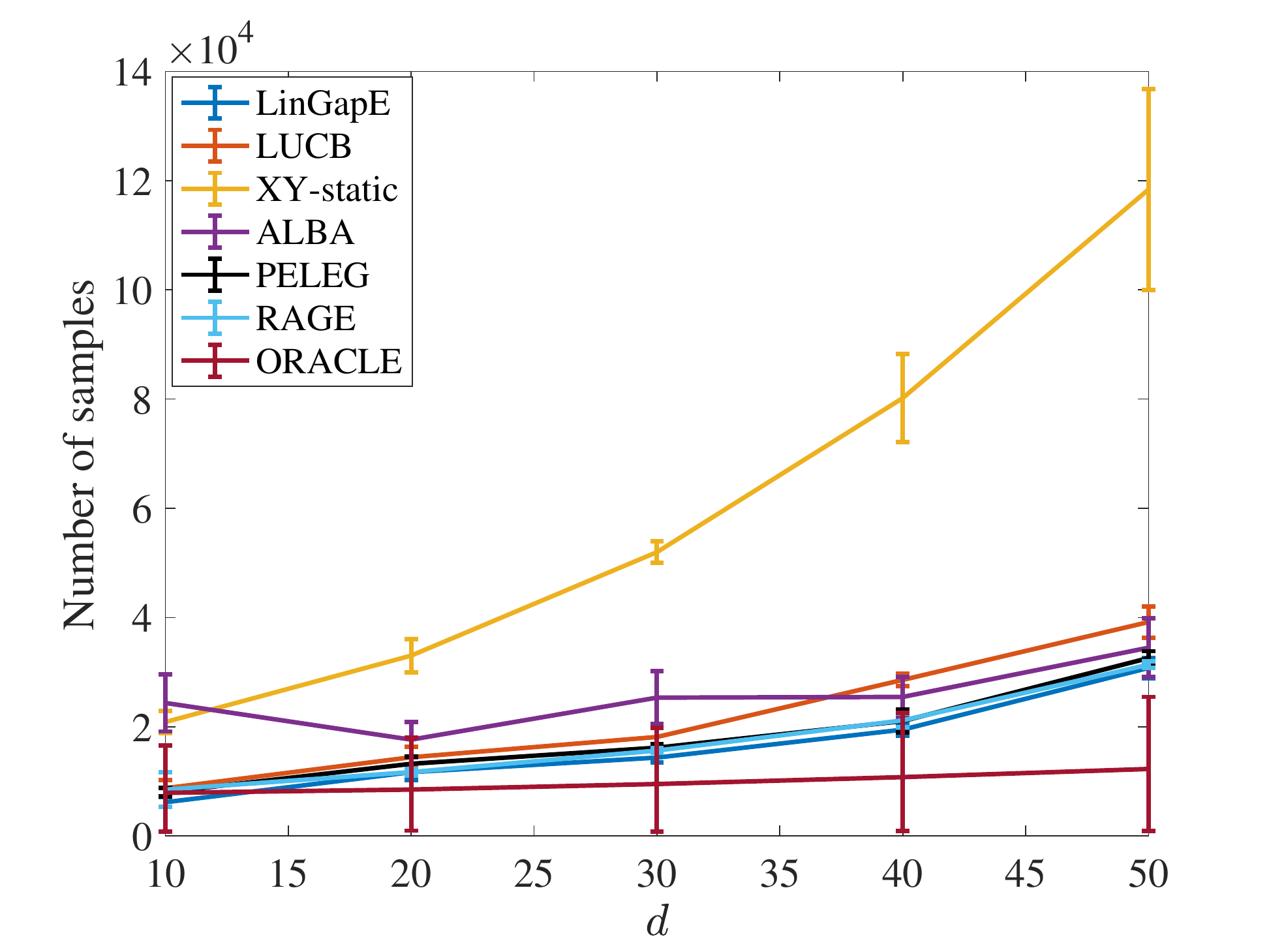}
		%\captionof{figure}{Dataset 2: 100 Arms sampled u.a.r from the surface of $\mathbb{S}^{d-1}, \gamma=0.01$.}
		%	\captionof{figure}{(a)}
		\label{fig:Dataset uniform}
		%	\end{figure}
	}
	\quad 
	\hspace*{0.2cm}
	\parbox{0.21\linewidth}{
		%\centering
		%\hspace*{-1cm}
		%\vspace{1.5cm}
		%\caption{Synthetic dataset 2,  100 trials, $\delta_{min}>0.05$}
		%\label{Synthetic dataset 2}
		%	\begin{figure}
		
		\includegraphics[scale=0.2]{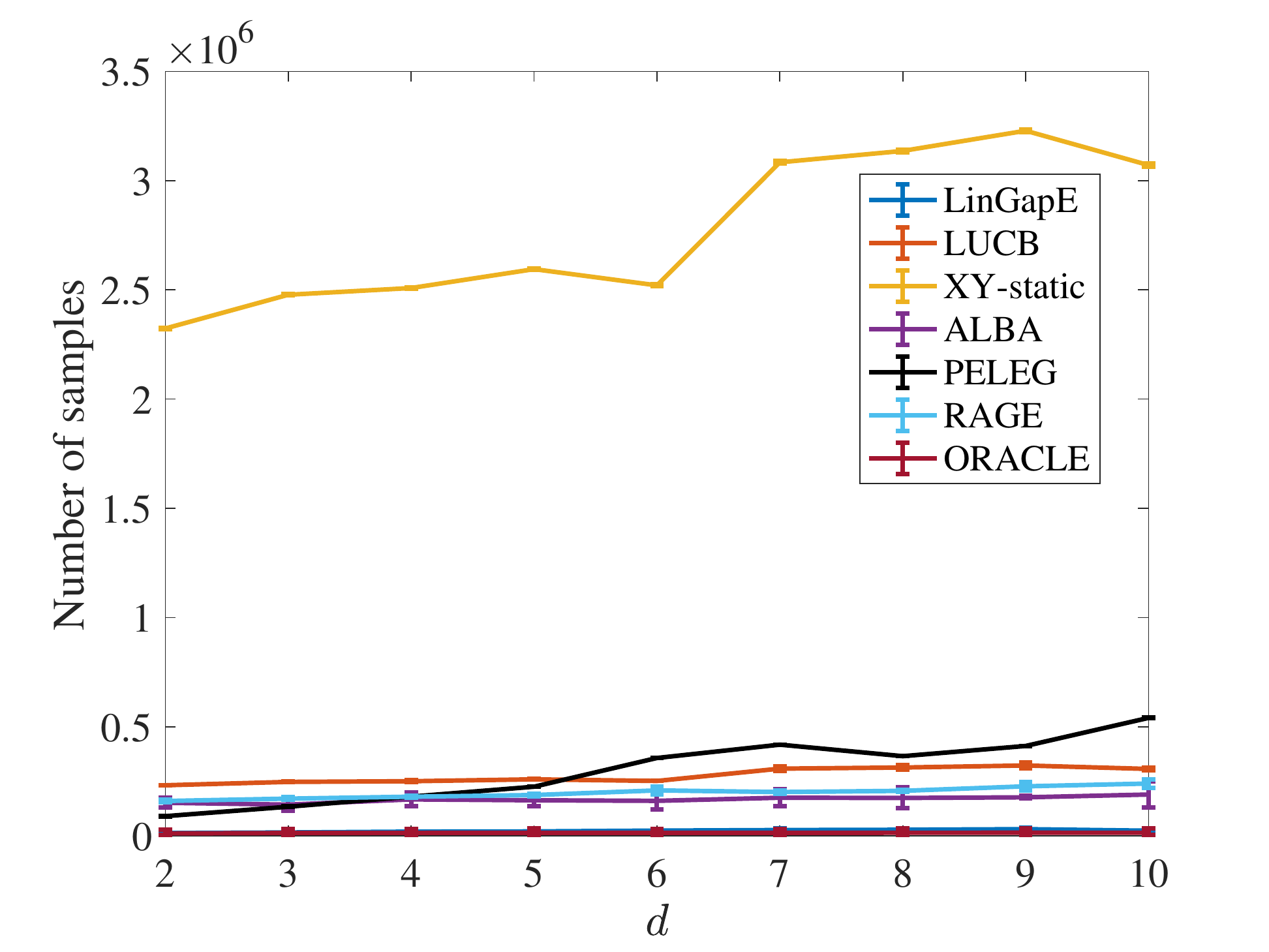}
		%\captionof{figure}{Dataset 3: Angle $ \omega = 0.1$ rad.}
		%\captionof{figure}{(a)}
		\label{fig:icml2errorbar-eps-converted-to}
		
		%\captionof{figure}{}
		%\label{fig:ICML2}
		%	\end{figure}
	}
	%\label{fig:Experiments}
	\vspace*{-0.2cm}
	\captionof{figure} {Sample complexity performance of linear bandit best arm identification algorithms for $3$ different settings: Standard bandit (Figs. a, b), Unit Sphere (Fig. c) and Standard bandit with confounding arm (Fig. d).} 
	\label{fig:Experiments}
	
	%
	%
	%Dataset 1: PELEG performs better than the other algorithms, Dataset 1:A closer look in the regime $ \Delta = \{0.11,\ldots,0.19\},$ comparison between PELEG and RAGE . Dataset 2: 100 Arms sampled u.a.r from the surface of $\mathbb{S}^{d-1}, \gamma=0.01$. Dataset 3: Angle $ \omega = 0.1$ rad.}
	
\end{table}

\section{Concluding Remarks  }
We have proposed a new, explicitly described algorithm for best arm identification in linear bandits, using tools from game theory and no-regret learning to solve minimax games. Several interesting directions remain unexplored. Removing the less-than-ideal dependence on the feature $C$ of the arm geometry and the extra logarithmic dependence on $\log(1/\delta)$ are perhaps the most interesting technical questions. It is also of great interest to see if a more direct game-theoretic strategy, along the lines of \cite{degenne-etal19non-asymptotic-exploration-solving-games}, exists for structured bandit problems, as also whether one can extend this machinery to solve for best policies in more general Markov Decision Processes.

%\newpage

\textbf{Broader Impact.} This work is largely theoretical in its objective. However, the problem that it attempts to lay sound theoretical foundations for is a widely encountered search problem based on features in machine learning. As a result, we anticipate that its implications may carry over to domains that involve continuous, feature-based learning, such as attribute-based recommendation systems, adaptive sensing and robotics applications. Proper care must be taken in such cases to ensure that recommendations or decisions from the algorithms set out in this work do not transgress considerations of safety and bias. While we do not address such concerns explicitly in this work, they are important in the design and operation of automated systems that continually interact with human users.

%
%
%Attribute based recommendation systems
%
%Online content recommendation

% % % % % % % % % % % % % % % % % % % % % % % % % % % % % % % % % % % % % %

\bibliography{ALS}

% % % % % % % % % % % % % % % % % % % % % % % % % % % % % % % % % % % % % %
\newpage
\appendix
\section{Glossary of symbols}\label{sec:glossaryOfSymbols}
\begin{enumerate}
    \item $\mathcal{A}^{MAX}_m:$ the EXP-WTS algorithm, used to compute the mixed strategy of the $MAX$ player in each round of PELEG.
    \item $a^*:$ the index of the best arm, i.e., $a^*:=\argmax_{i\in[K]}x_i^T\theta^*$.
    \item $B(0,D_m):$ the \emph{closed} ball of radius $D_m$ in $\reals^d$, centered at $0$.
    \item $C = \lambda_{\min}\left(\sum_{x\in\mX}xx^T\right)$.
    \item $\mathcal{C}_m(x):= \left\{\lambda\in \Real^d: \exists x'\in \mX_m, x'\neq x| \lambda^Tx'\geq\lambda^Tx+\epsilon_m \right\}$ is the union of all hyperplanes $\lbrace \lambda\in\reals^d| \lambda^T(x'-x)\geq\epsilon_m \rbrace.$
    \item $D_m:=2(\sqrt{2}-1)\sqrt{\frac{{C}}{{\max\limits_{x,x'\in \mX_m, x\neq x'}\norm{x-x'}_2^2\log K}}}$.
    \item $d:$ dimension of space in which the feature vectors $x_1,\cdots,x_K$ reside.
    \item $\Delta_i = (x^*-x_i)^T\theta^*,~i\neq a^*.$
    \item $\Delta_{\min} = \min_{i\neq a^*}\Delta_i.$
    \item $\delta:$ maximum allowable probability of erroneous arm selection (a.k.a confidence parameter).
    \item $\delta_m = \frac{\delta}{m^2}.$
    \item $\mathcal{E}(0,V,r):=\lbrace \lambda\in\reals^d\mid \lambda^TV\lambda\leq r^2 \rbrace$, is the confidence ellipsoid with center $0$, shaped by $V$ and $r.$
    \item $\mathcal{H}(x,x'):$
    \item $K = |\mX|$ number of feature vectors.
    \item $N_m:$ the length of Phase~$m.$
    \item $\nu_k:$ rewards from Arm~$k$ are all drawn IID from $\nu_k.$
    \item $\mP(\Omega):=\lbrace p\in[0,1]^{|\Omega|}:\parallel p\parallel_1=1 \rbrace$, the set of all probability measures on some given set $\Omega.$
    \item $r_m = \sqrt{8\log{\frac{K^2}{\delta_m}}}.$
    \item $\theta^*:$ fixed but unknown vector in $\reals^d$ that parameterizes the means of $\nu_k,$ i.e., the mean of $\nu_k$ is $x_k^T\theta^*.$
    \item $n_t^k:$ number of times Arm~$k$ has been sampled up to Round~$t$ of PELEG.
    \item $\hat{\theta}_m:$ OLS estimate of $\theta^*$ at the end of Phase~$m$ of PELEG.
    \item $V_t^m = \sum_{s\leq t}x_sx_s^T$ the design matrix in Round~$t$ of Phase~$m.$
    \item $W_t = \sum_{x\in\mX}w_xxx^T$  the design matrix formed by sampling arms$\sim w\in\mP(\mX).$
    \item $\mathcal{X}=\{x_1,\cdots,x_K\}$, the feature set.
    \item $\mathcal{X}_m$ the set of features that survive Phase~$m$ of PELEG.
\end{enumerate}

\section{Technical lemmas}\label{appendix:technical lemmas}

\subsection{Tracking lemma}
The $\mathcal{A}^{MAX}$ subroutine recommends a distribution $w_t$ in every round $t$
 over the set of arms. In order to play an arm from this distribution we use a \enquote{tracking} rule, which helps the number of arm pulls to stay close to the cumulative sum $\sum\limits_{s=1}^{t}w_s^k$ for each arm $k\in [K]$.

 \begin{lemma}[$n^k_t$ tracks $\sum_{s=1}^tw^k_s$]\label{lemma:tracking}
In any phase $m\geq 1$, if for $\forall t\geq K$, the following strategy for pulling the arms is used.
\[ \text{Choose arm, } k_t=\argmin\limits_{k\in [K]} \frac{n_{t-1}^k}{\sum\limits_{s=1}^{t}w_s^k},\]
then, for all $t\geq K$ and for all $k\in [K]$, $ \sum\limits_{s=1}^{t}w_s^k-\left(K-1\right)\leq n_t^k\leq \sum\limits_{s=1}^{t}w_s^k +1 $. 
\end{lemma}
\begin{proof}\label{proof:tracking}
We first show the upper-bound. We need to show that the inequality holds for all arms. First, let $k\neq k_t$. We will use induction on $t$. 
\newline {Base case:} At $t=K$, $n_t^j=1=\sum\limits_{s=1}^{t}w_s^j, \forall j\in [K]$.
\newline Let, the induction hold for all $s < t$. We will show that the inequality holds for $t$. If $k\neq k_t$, then $n_t^k=n_{t-1}^k\stackrel{(*)}{\leq} \sum\limits_{s=1}^{t-1}w_s^k +1 \leq \sum\limits_{s=1}^{t}w_s^k +1$.
\newline Next, let $k=k_t$. We note that by definition of $k_t$, we have 
\[\frac{n_{t-1}^{k_t}}{\sum\limits_{s=1}^{t}w_s^{k_t}}\stackrel{(*)}{\leq} \frac{\sum\limits_{j=1}^{K}n_{t-1}^{j}}{\sum\limits_{j=1}^{K}\sum\limits_{s=1}^{t}w_s^{j}}=\frac{t-1}{t} \leq 1. \]
Here, the inequality (*) follows because of the following fact: 
{for any sequence of positive numbers $\{a_i\}_{1\leq i\leq n}$ and $\{b_i\}_{1\leq i\leq n}$}, $\min\limits_{1\leq i\leq n} \frac{a_i}{b_i} \leq \frac{\sum_{i=1}^{n}a_i}{\sum_{i=1}^{n}b_i}. $
Consequently, $\frac{n_{t}^{k_t}}{\sum\limits_{s=1}^{t}w_s^{k_t}}= \frac{n_{t-1}^{k_t}+1}{\sum\limits_{s=1}^{t}w_s^{k_t}} \leq 1+ \frac{1}{\sum\limits_{s=1}^{t}w_s^{k_t}}$. Rearranging completes the proof of the right hand side.
\newline For the lower bound inequality, we observe that for any $k\in [K]$,
\[n_t^k =t-\sum\limits_{j\neq k}n_t^j \stackrel{(*)}{\geq} t-\sum\limits_{j\neq k}\sum\limits_{s=1}^{t}w_s^j-\left(K-1\right) = t-\sum\limits_{j=1}^{K}\sum\limits_{s=1}^{t}w_s^j +\sum\limits_{s=1}^{t}w_s^k-\left(K-1\right)= \sum\limits_{s=1}^{t}w_s^k-\left(K-1\right).\]
Here, the inequality $(*)$ follows from the the upper-bound on $n_t^k,~k\in [K]$.
\end{proof}

\subsubsection{Details of $\mathcal{A}^{MAX}_m$ (EXP-WTS)}
We employ the EXP-WTS algorithm to recommend to the MAX player, the arm to be played in round $t> K$. At the start of every phase $m\geq 1$, an EXP-WTS subroutine is instantiated afresh, with initial weight vectors to be $1$ for each of the $K$ experts. The $K$ experts are taken to be standard unit vectors $(0,0,\ldots,0,1,0,\ldots,0)$ with $1$ at the $k^{th}$ position, $k\in [K]$. The EXP-WTS subroutine recommends an exponentially-weighted probability distribution over the number of arms, depending upon the weights on each expert. The loss function supplied to update the weights of each expert, is indicated in Step~\ref{PELEGstep:passLoss2ExpWts} of Algorithm \ref{alg:PEPEG-mixed}.

 EXP-WTS requires a bound on the losses (rewards) in order to set its learning parameter optimally. This is ensured by passing an upper-bound of $D_m^2$ ($\because$ in any Phase~$m,\norm{\lambda}_2\leq D_m,$ see Step~\ref{PELEGStep:phaseStoppingCriterion} of Algorithm~\ref{alg:PEPEG-mixed}).

\begin{lemma}\label{lemma:regret of MAX player}
In any phase $m$, at any round $t>K$, $\mathcal{A}^{MAX}_m$ has a regret bounded as 
\[R_t \leq \frac{D_m^2}{\sqrt{2}-1}\sqrt{t\log K}.\] 
\end{lemma}

\begin{proof}
	The proof involves a simple modification of the proof of the regret analysis of the EXP-WTS algorithm (see for example, \cite{cesa-bianchi06prediction-learning-games}), with loss scaled by $[0, D_m^2]$ followed by the well-known \emph{doubling trick.}
\end{proof}

% % % % % % % % % % % % % % % % % % % % % % % % % % % % % % % % % % % % % % % %
\section{Proof of Key Lemma}\label{appendix:Proof of key lemma}

% \begin{lemma}[Key Lemma]\label{lemma:key lemma on uncertainity after phase m}
% At the end of each phase $m\geq 1$,
% \[\max\limits_{x,x'\in \mX_m, x\neq x'} \norm{x-x'}^2_{\left({V_{N_m}^m}\right)^{-1}} \leq \frac{\left(\left(\frac{1}{2}\right)^{m+1}\right)^2}{8\log K^2/\delta_m}. \]
% \end{lemma}

\keylemma*

\begin{proof}
Let $r_m:=\sqrt{8\log K^2/\delta_m}$, for ease of notation.
The phase stopping criterion is 
\begin{equation}\label{eq:phase stopping criterion}
\text{STOP at round $ t \geq K $ if: }\min\limits_{\lambda\in \bigcup\limits_{x\in \mX_m}\mC_m(x)\cap B(0,D_m)} \norm{\lambda}^2_{\left({V_{N_m}^m}\right)} >r_m^2. 
\end{equation}

Note that the set $\mC_m(x)$ depends on the value that $\epsilon_m$ takes in phase $m$. Depending on the value of $\epsilon_m$, we divide the analysis into the following two cases. 
\subsection*{Case 1. $\epsilon_m=\left(1/2\right)^{m+1} $.}
In this case $\frac{D_m\sqrt{C}}{r_m} \geq 1 $. For any phase $m\geq 1$, and  $t\geq 1$, let us define the ellipsoid $\mE\left(0, V_t^m, r_m\right)  := \left\{\theta\in \Real^d:\norm{\theta}_{{V_t^m}}^2\leq r_m^2  \right\}$.
The phase stopping rule at round $t\geq K$ is equivalent to :
\begin{align*}
\quad \quad\text{STOP if}~&:& \mE(0, V_t^m, r_m) \bigcap \left\{\bigcup\limits_{x\in \mX_m}\mC_m(x)\cap B(0,D_m) \right\} = \emptyset \text{ (empty set)}\\
&\Leftrightarrow&   \left\{\mE(0, V_t^m, r_m)\cap B(0,D_m) \right\} \bigcap \left\{\bigcup\limits_{x\in \mX_m}\mC_m(x) \right\} = \emptyset.
\end{align*}
However by Rayleigh' inequality\footnote{for any PSD matrix $A$ and $x\in \Real^d$, $\lambda_{min}(A)\leq \frac{x^TAx}{x^Tx}\leq \lambda_{max}(A)$} followed by the fact that $\frac{D_m\sqrt{C}}{r_m} \geq 1 $, we have for any $\theta \in \mE(0, V_t^m,r_m) $, 

\[\norm{\theta}_2^2\leq \frac{\norm{\theta}_{V_t^m}^2}{\lambda_{min}(V_t^m)}\stackrel{(*)}{\leq} \frac{\norm{\theta}_{V_t^m}^2}{\lambda_{min}(\sum\limits_{k=1}^{K}x_kx_k^T )}\leq \frac{r_m^2}{C}\leq D_m^2.\]

The inequality (*) follows from the following fact: for $t\geq K$, $V_t^m=\sum\limits_{k=1}^{K}x_kx_k^T+ \sum\limits_{s=K+1}^{t}x_sx_s^T \succcurlyeq \sum\limits_{k=1}^{K}x_kx_k^T$. 

$\therefore \mE(0,V_t^m,r_m)\subseteq B(0,D_m), \forall t\geq K$. Hence the phase stopping rule reduces to,
\begin{align*}
\text{ STOP if: } \mE(0, V_t^m, r_m) \bigcap  \left\{\bigcup\limits_{x\in \mX_m}\mC_m(x) \right\} =\emptyset& \Leftrightarrow \min\limits_{\lambda \in \cup_{x\in \mX_m}\mC_m(x) } \norm{\lambda}_{V_t^m}^2 >r_m^2\\ &\Leftrightarrow \min\limits_{\lambda \in \bigcup\limits_{(x,x')\in \mX_m^2}\{\lambda': {\lambda'}^Tx'\geq {\lambda'}^Tx+\left(1/2\right)^{m+1}\} } \norm{\lambda}_{V_t^m}^2 >r_m^2.
\end{align*}
The above reduction is a minimization problem over union of halfspaces. For any fixed pair $\left(x,x'\right)\in \mX_m^2, x\neq x'$, this is a quadratic optimization problem with linear constraints, which can be explicitly solved using standard Lagrange method. 
\begin{lemma}[Supporting Lemma for Lem.~\ref{lemma:key lemma on uncertainity after phase m}]\label{lemma:equivalence of quadratic opts}
	For any two arms $x$ and $x'$, we have that
	\begin{align*}
	\min\limits_{\lambda \in \{\lambda': {\lambda'}^Tx'\geq {\lambda'}^Tx+\left(\frac{1}{2}\right)^{m+1}\} } \norm{\lambda}_{V_t^m}^2 =  \frac{\left(\left(\frac{1}{2}\right)^{m+1}\right)^2}{\norm{x-x'}^2_{(V_t^m)^{-1}}}.
	\end{align*}
\end{lemma}

\begin{proof}
	The result follows by solving the optimization problem explicitly using the Lagrange multiplier method.
\end{proof}

By using the above lemma  we obtain:
\[\text{STOP if:} \forall x,x'\in \mX_m, x\neq x', \norm{x-x'}^2_{(V_t^m)^{-1}}< \frac{\left(\left(\frac{1}{2}\right)^{m+1}\right)^2}{8\log K^2/\delta_m}. \]
Hence, at round $t=N_m$ we have, $\forall x,x'\in \mX_m, x\neq x', \norm{x-x'}^2_{(V_{N_m}^m)^{-1}}< \frac{\left(\left(\frac{1}{2}\right)^{m+1}\right)^2}{8\log K^2/\delta_m}$.

\subsection*{Case 2. $\epsilon_m=\frac{D_m\sqrt{C}}{r_m}\left(\frac{1}{2}\right)^{m+1}  $.}
In this case, we have $\frac{D_m\sqrt{C}}{r_m} < 1$.
\newline The phase ends when $\forall (x,x')\in \mX_m^2$, $\min\limits_{\substack{{\lambda\in \left\{\lambda\in \Real^d:\lambda^Tx'\geq \lambda^Tx+\epsilon_m \right\}}{\cap B(0,D_m)}}}\norm{\lambda}^2_{V_t^m}>r_m^2$. 
%We will show that when the above condition holds for any pair of arms $(x,x')\in \mX_m^2$, $\norm{x-x'}^2_{\left({V_{N_m}^m}\right)^{-1}} \leq \frac{\left(\left(\frac{1}{2}\right)^{m+1}\right)^2}{8\log K^2/\delta_m}$. 
Let us decompose the optimization problem defining the phase stopping criteria into smaller sub-problems, depending on  pair of arms $(x,x')$ in $\mX_m^2$. That is, we split the set $\cup_{x\in \mX_m}\mC_m(x)$ in equation (\ref{eq:phase stopping criterion}), and  consider the following problem: for any pair of distinct arms $(x,x')\in \mX_m$, consider 
\[P(x,x'): \min\limits_{\substack{{\lambda\in \left\{\lambda\in \Real^d:\lambda^Tx'\geq \lambda^Tx+\epsilon_m \right\}}{\cap B(0,D_m)}}}\norm{\lambda}^2_{V_t^m}. \]
let $t_{x,x'}$ be the first round when $\min\limits_{\substack{{\lambda\in \left\{\lambda\in \Real^d:\lambda^Tx'\geq \lambda^Tx+\epsilon_m \right\}}{\cap B(0,D_m)}}}\norm{\lambda}^2_{V_t^m}>r_m^2$.
Clearly, we have $N_m=\max\limits_{(x,x')\in \mX_m^2, x\neq x'} t_{x,x'}$.
In addition, for any $t\geq t_{x,x'}$, $\norm{\lambda}_{V_t^m}^2=\lambda^T\left(V^m_{t_{x,x'}} +\sum_{s=t_{x,x'}+1}^{t}x_sx_s^t\right)\lambda=\norm{\lambda}^2_{V^m_{t_{x,x'}}}+ \sum_{s=t_{x,x'}+1}^{t}(x_s^T\lambda)^2\geq \norm{\lambda}^2_{V^m_{t_{x,x'}}} >r_m^2$. Hence, once the inequality for a given pair of arms $(x,x')$ is fulfilled it is satisfied for all subsequent rounds.  We will now analyze the problem $P(x,x')$ for each pair of arms $(x,x')\in \mX_m^2$ individually.
\par
For any $t\geq 1$ , define $\lambda_{t}^*\in \argmin\limits_{\substack{{\lambda\in \left\{\lambda\in \Real^d:\lambda^Tx'\geq \lambda^Tx+\epsilon_m \right\}}\\{\cap B(0,D_m)}}}\norm{\lambda}^2_{V_t^m}$. Note that $\lambda_t^*$ is specific to the pair  $(x,x')$. 
\subsubsection*{Claim 1. ${\lambda_{t}^*}^T(x'-x)=\epsilon_m, \forall t\geq 1$.}
\begin{proof}[Proof of Claim 1]
For the proof, let's denote $\lambda^*\equiv \lambda_{t}^*$. Now, suppose that the claim was not true, i.e., ${\lambda^*}^T(x'-x)=\epsilon_m+a$ for some $a>0$. Let $b= \frac{a}{{\lambda^*}^T(x'-x)}$. Then $0<b<1$. Define $\lambda':= (1-b)\lambda^*$. By construction, ${\lambda'}^T(x'-x)=\epsilon_m$, and $\norm{\lambda'}_2=(1-b)\norm{\lambda^*}_2<\norm{\lambda^*}_2$. Hence $ \lambda' \in { \left\{\lambda\in \Real^d:\lambda^Tx'\geq \lambda^Tx+\epsilon_m \right\}}{\cap B(0,D_m)} $. However, $\norm{\lambda'}_{V_t^m}=(1-b)\norm{\lambda^*}_{V_t^m}<\norm{\lambda^*}_{V_t^m}$, which is a contradiction. 
\end{proof}

At $t=t_{x,x'},$ we have $\min\limits_{\substack{{\lambda\in \left\{\lambda\in \Real^d:\lambda^Tx'\geq \lambda^Tx+\epsilon_m \right\}}\\{\cap B(0,D_m)}}}\norm{\lambda}^2_{V_t^m}>r_m^2$. We have two sub-cases depending on the 2-norm of $\lambda_t^*$.
\subsubsection*{Sub-case 1. $ \norm{\lambda_t^*}_2<D_m $.}
In this case, we have the following equivalence:
\[ \min\limits_{\substack{{\lambda\in \left\{\lambda\in \Real^d:\lambda^Tx'\geq \lambda^Tx+\epsilon_m \right\}}\\{\cap B(0,D_m)}}}\norm{\lambda}^2_{V_t^m} \equiv \min\limits_{\substack{{\lambda\in \left\{\lambda\in \Real^d:\lambda^Tx'\geq \lambda^Tx+\epsilon_m \right\}}}}\norm{\lambda}^2_{V_t^m}.  \]
This can be seen by noting that if $\norm{\lambda_t^*}_2<D_m$, then the corresponding Lagrange multiplier is zero. Hence at round $t=t_{x,x'}$, by solving a standard Lagrange optimization problem, we get $\norm{x-x'}^2_{(V_t^m)^{-1}}< \frac{\epsilon_m^2}{8\log K^2/\delta_m} = \frac{D_m^2C}{r_m^2} \frac{\left(\frac{1}{2}\right)^{2(m+1)}}{8\log K^2/\delta_m}< \frac{\left(\frac{1}{2}\right)^{2(m+1)}}{8\log K^2/\delta_m}$. The last inequality follows from the hypothesis of Case 2. Since $N_m\geq t_{x,x'}$, we get $\norm{x-x'}^2_{(V_{N_m}^m)^{-1}}\leq \norm{x-x'}^2_{\left(V^m_{t_{x,x'}}\right)^{-1}}< \frac{\left(\left(\frac{1}{2}\right)^{m+1}\right)^2}{8\log K^2/\delta_m}$. 
\subsubsection*{Sub-case 2. $ \norm{\lambda_t^*}_2=D_m $.}
The sub-case when $\norm{\lambda_t^*}_2=D_m$, is more involved. Let's enumerate the properties of $\lambda_t^*$  at $t=t_{x,x'}$ that we have.
\begin{itemize}
\item $\norm{\lambda_t^*}_{V_t^m}^2>r_m^2.$
\item $\norm{\lambda_t^*}_2=D_m.$
\item ${\lambda_{t}^*}^T(x-x')=\epsilon_m. $ 

\end{itemize} 
 We divide the analysis of this sub-case into two further sub-cases. 
\subsubsection*{Sub-sub-case 1. $ r_m^2 \norm{x-x'}^2_{(V_t^m)^{-1}}>\epsilon_m^2 $.}

Let $\theta_t^*:=\argmax\limits_{\theta\in \mE(0,V_t^m,r_m)} \theta^T(x'-x)$. Then, one can verify by solving the maximization problem explicitly that ${\theta_t^*}^T(x'-x)=r_m\norm{x'-x}_{(V_t^m)^{-1}}$.  Let $\theta_1:= \frac{{\theta_t^*}^T(x'-x)}{\norm{x'-x}_2^2}(x'-x)$. We have the following properties of $\theta_1$ by construction, which are straight-forward to verify.
\begin{itemize}
\item $\norm{\theta_1}_2 = \frac{r_m\norm{x'-x}_{(V^m_t)^{-1}}}{\norm{x'-x}_2}.$
\item $\theta_1^T(\theta_t^*-\theta_1)=0. $
\end{itemize}
Let $\lambda_1:=\frac{{\lambda_t^*}^T(x'-x)}{\norm{x'-x}_2^2}(x'-x)$. It follows that, $\norm{\lambda_1}_2=\frac{\abs{{\lambda_t^*}^T(x'-x)}}{\norm{x'-x}_2}=\frac{\epsilon_m}{\norm{x'-x}_2}$.
\newline Finally, let us define two more quantities. Let $\lambda_2:=\frac{r_m\norm{x'-x}_{(V_t^m)^{-1}}}{\epsilon_m}\lambda_t^*$ and $\theta_2:=\frac{\epsilon_m}{r_m\norm{x'-x}_{(V^m_t)^{-1}}}\theta_t^*$.
We have by the hypothesis of sub-sub-case 1, that $\norm{\theta_2}_2^2<\norm{\theta_t^*}_2^2$. This implies that $\theta_2\in \mE(0,V_t^m,r_m)$.

Next, we make the following two claims on the 2-norms of $\theta_2$ and $\theta_t^*-\theta_1$.
\subsubsection*{Claim. $ \norm{\theta_2}_2>D_m$.}
\begin{proof}[Proof of Claim.]
Suppose that $\theta_2\in B(0,D_m)$. By construction, $\theta_2^T(x'-x)=\epsilon_m$. Hence, $\theta_2\in \left\{\lambda\in \Real^d:\lambda^Tx'\geq \lambda^Tx+\epsilon_m \right\}\cap B(0,D_m). $
Since, $\theta_2 \in \mE(0,V_t^m,r_m)$, this implies that $\norm{\theta_2}_{V_t^m}\leq r_m$. However, this is a contradiction since at round $t$, $\min\limits_{\substack{{\lambda\in\left\{\lambda^Tx'\geq \lambda^Tx+\epsilon_m \right\} }\\{\cap B(0,D_m)}}}>r_m^2$. 
\end{proof}
Hence, we have the following,
\[D_m^2<\norm{\theta_2}_2^2=\frac{\epsilon_m^2}{r_m^2\norm{x'-x}^2_{(V_t^m)^{-1}}}\norm{\theta_t^*}_2^2= \frac{D_m^2}{\norm{\lambda_2}_2^2}\norm{\theta_t^*}_2^2\Rightarrow  \norm{\theta_t^*}_2^2 >\norm{\lambda_2}_2^2.  \]

\subsubsection*{Claim. $ \norm{\theta_t^*-\theta_1}_2^2 > \norm{\lambda_2-\theta_1}_2^2 $.}
\begin{proof}[Proof of Claim.]
First we note that,
\begin{align*}
\theta_1^T(\theta_t^*-\lambda_2) &= \frac{{\theta_t^*}^T(x'-x)}{\norm{x'-x}_2^2}(x'-x)^T\left(\theta_t^*- \frac{r_m\norm{x'-x}_{(V^m_t)^{-1}}}{\epsilon_m}\lambda_t^*\right)\\
&= \frac{r_m^2 \norm{x'-x}^2_{(V_t^m)^{-1}}}{\norm{x'-x}_2^2} - \frac{r_m^2 \norm{x'-x}^2_{(V_t^m)^{-1}}}{\norm{x'-x}_2^2} = 0.
\end{align*}
Next observe that,
\begin{align*}
\norm{\theta_t^*-\theta_1}_2^2 &= \norm{\theta_t^*}_2^2 + \norm{\theta_1}_2^2 -2{\theta_t^*}^T\theta_1\\
&= \norm{\theta_t^*}_2^2 + \norm{\theta_1}_2^2 -2({\theta_t^*-\lambda_2})^T\theta_1 -2\theta_1^T\lambda_2\\
&=\norm{\theta_t^*}_2^2 + \norm{\theta_1}_2^2  -2\theta_1^T\lambda_2\\
&> \norm{\lambda_2}_2^2 + \norm{\theta_1}_2^2  -2\theta_1^T\lambda_2 = \norm{\lambda_2-\theta_1}_2^2.
\end{align*}
\end{proof}
Putting things together we have,
\begin{align*}
\norm{\theta_t^*}_2^2 &= \norm{\theta_t^*-\theta_1}_2^2 + \norm{\theta_1}_2^2\\
\Rightarrow  \norm{\theta_1}_2^2 &= \norm{\theta_t^*}_2^2 - \norm{\theta_t^*-\theta_1}_2^2 \\
\Rightarrow  \norm{\theta_1}_2^2 &< \norm{\theta_t^*}_2^2 - \norm{\lambda_2-\theta_1}_2^2 \\
\Rightarrow  \frac{r_m^2\norm{x'-x}^2_{(V_t^m)^{-1}}}{\norm{x'-x}_2^2} &< \frac{r_m^2}{C} - r_m^2\norm{x'-x}^2_{(V_t^m)^{-1}} \left(\frac{D_m^2}{\epsilon_m^2}-\frac{1}{\norm{x'-x}_2^2} \right)\\
\Rightarrow  \frac{\norm{x'-x}^2_{(V_t^m)^{-1}}}{\norm{x'-x}_2^2} &< \frac{1}{C} - \norm{x'-x}^2_{(V_t^m)^{-1}} \left(\frac{D_m^2}{\epsilon_m^2}-\frac{1}{\norm{x'-x}_2^2} \right)\\
\Rightarrow  \frac{\norm{x'-x}^2_{(V_t^m)^{-1}}}{\norm{x'-x}_2^2} &< \frac{1}{C} - \norm{x'-x}^2_{(V_t^m)^{-1}} \frac{D_m^2}{\epsilon_m^2}+\frac{\norm{x'-x}^2_{(V_t^m)^{-1}}}{\norm{x'-x}_2^2}\\
\Rightarrow  \norm{x'-x}^2_{(V_t^m)^{-1}} & < \frac{\epsilon_m^2}{D_m^2C}= \frac{D_m^2C}{r_m^2D_m^2C} \left(\frac{1}{2}\right)^{2(m+1)}=\frac{\left(\left(\frac{1}{2}\right)^{m+1}\right)^2}{8\log K^2/\delta_m}.
\end{align*}

\subsubsection*{Sub-sub-case 2. $ r_m^2 \norm{x-x'}^2_{(V_t^m)^{-1}}\leq\epsilon_m^2 $. }
This case is trivial as by the hypothesis, 
\[\norm{x-x'}^2_{(V_t^m)^{-1}}\leq \frac{\epsilon_m^2}{r_m^2}= \frac{D_m^2C}{r_m^2}\frac{1}{r_m^2}\left(\left(\frac{1}{2}\right)^{m+1}\right)^2< \frac{\left(\left(\frac{1}{2}\right)^{m+1}\right)^2}{8\log K^2/\delta_m}.  \]
This completes the proof of the key lemma.

\end{proof}

% % % % % % % % % % % % % % % % % % % % % % % % % % % % % % % % % % % % % % %
\section{Proofs of bounds on phase length}\label{appendix:Proofs of bounds on phase length}
In this section we will provide an upper-bound on the length of any phase $m\geq 1$. Clearly, the length of any phase $m$ is governed by the value of $\epsilon_m$ in that phase. Towards this, we have the following lemma. 

\phaselengthlemma*

% \begin{lemma}\label{lemma: bound on phase length}
% Let $B_m:= \min\limits_{w\in \mathcal{P}_K}\max\limits_{x,x'\in \mX_m, x\neq x'} \norm{x-x'}^2_{W^{-1}}$.
% There exists $\delta_0$ such that $\forall \delta<\delta_0$, the length $N_m$ of any phase $m$ is bounded as :
% \begin{align*}
% N_m \leq \begin{cases}
%   2B_m\left(2^{m+1}\right)^2\left[\frac{r_m^4\log K}{(\sqrt{2}-1)^2C^2} \right]+1 &\text{  if } \epsilon_m = \frac{D_m\sqrt{C}}{r_m}\left(\frac{1}{2}\right)^{m+1},\\
%  2B_m\left(2^{m+1}\right)^2r_m^2+1&\text{  if } \epsilon_m = \left(\frac{1}{2}\right)^{m+1}.
% \end{cases}
% \end{align*}

% \end{lemma}
\begin{proof}
Recall that $r_m = \sqrt{8\log{K^2}/\delta_m}.$
Let $t$ be the last round in phase $m$, \textit{before} the phase ends. Then by definition of phase stopping rule (Step~12 of the algorithm),
\begin{align*}
r_m^2 &\geq \min\limits_{\lambda\in \bigcup\limits_{x\in \mX_m}\mC_m(x)\cap B(0,D_m)} \norm{\lambda}^2_{{V_t^m}}\\
&\stackrel{(i)}\geq \min\limits_{\lambda\in \bigcup\limits_{x\in \mX_m}\mC_m(x)\cap B(0,D_m)} \sum\limits_{s=1}^{t}\norm{\lambda}^2_{{W_s}}-K^2D_m^2\\
&\stackrel{(ii)}\geq \sum\limits_{s=1}^{t}\norm{\lambda_s}^2_{{W_s}}-K^2D_m^2\\
&\stackrel{(iii)}= \sum\limits_{s=1}^{t}\sum\limits_{k=1}^{K}w_s^k\left(\lambda_s^Tx_k\right)^2-K^2D_m^2\\
&\stackrel{(iv)}\geq \max\limits_{w\in \mathcal{P}_K} \sum\limits_{s=1}^{t}\sum\limits_{k=1}^{K}w^k\left(\lambda_s^Tx_k\right)^2-\frac{D_m^2}{\sqrt{2}-1}\sqrt{t\log K} - K^2D_m^2\\
&= \max\limits_{w\in \mathcal{P}_K} \sum\limits_{s=1}^{t}\norm{\lambda_s}^2_{W}-\frac{D_m^2}{\sqrt{2}-1}\sqrt{t\log K} - K^2D_m^2\\
&= t.\max\limits_{w\in \mathcal{P}_K} \sum\limits_{s=1}^{t}\frac{1}{t}\norm{\lambda_s}^2_{W}-\frac{D_m^2}{\sqrt{2}-1}\sqrt{t\log K} - K^2D_m^2\\
&\stackrel{(v)}\geq t.\max\limits_{w\in \mathcal{P}_K}\min\limits_{q\in \mathcal{P}\left(\bigcup\limits_{x\in \mX_m}\mC_m(x)\cap B(0,D_m) \right)} \mathbb{E}_{\lambda\sim q}\left[\norm{\lambda}^2_{W} \right]- \frac{D_m^2}{\sqrt{2}-1}\sqrt{t\log K} - K^2D_m^2\\
&\stackrel{(vi)}\geq t.\max\limits_{w\in \mathcal{P}_K}\min\limits_{q\in \mathcal{P}\left(\bigcup\limits_{x\in \mX_m}\mC_m(x) \right)} \mathbb{E}_{\lambda\sim q}\left[\norm{\lambda}^2_{W} \right]- \frac{D_m^2}{\sqrt{2}-1}\sqrt{t\log K} - K^2D_m^2\\
&\stackrel{(vii)}= t\frac{\epsilon_m^2}{B_m}- \frac{D_m^2}{\sqrt{2}-1}\sqrt{t\log K} - K^2D_m^2.
\end{align*}
Here the inequalities follow because of (i) lemma \ref{lemma:tracking}, (ii) best-response of MIN player as given in Step~15 of the algorithm, (iii) by definition of $W_s$ in Step~14, (iv) regret property of MAX player (see lemma \ref{lemma:regret of MAX player}), (v) $\sum\limits_{s=1}^{t}\frac{1}{t}\mathbbm{1}\{\lambda=\lambda_s \} \in \mathcal{P}\left(\bigcup\limits_{x\in \mX_m}\mC_m(x)\cap B(0,D_m) \right)$, (vi) taking minimum over a larger set, and (vii) follows by explicitly solving the minimization problem and recalling the definition of $B_m$
We have that,
\begin{equation}\label{eq: phase length general equation}
t-\frac{B_m}{(\sqrt{2}-1)\epsilon_m^2}D_m^2\sqrt{\log K}\sqrt{t} \leq \frac{B_m}{\epsilon_m^2}r_m^2 + \frac{B_m}{\epsilon_m^2}K^2D_m^2. 
\end{equation}
We will do the analysis depending on the value that $\epsilon_m$ takes in phase $m$.
\subsection*{Case 1. $\epsilon_m = \frac{D_m\sqrt{C}}{r_m}\left(\frac{1}{2}\right)^{m+1}$.}
In this case we have, $\frac{D_m\sqrt{C}}{r_m} <1$. Applying the value of $\epsilon_m$ in eq. (\ref{eq: phase length general equation}), we have
\begin{align}
t-\frac{B_m}{(\sqrt{2}-1)\epsilon_m^2}D_m^2\sqrt{\log K}\sqrt{t} &\leq \frac{B_m}{\epsilon_m^2}r_m^2 + \frac{B_m}{\epsilon_m^2}K^2D_m^2\nonumber\\
\Rightarrow   t-\frac{B_m}{(\sqrt{2}-1)C}r_m^2\left(2^{m+1}\right)^2\sqrt{\log K}\sqrt{t} &\leq \frac{B_m}{D_m^2C}r_m^4\left(2^{m+1}\right)^2 + \frac{B_m}{C}r_m^2\left(2^{m+1}\right)^2K^2.\label{eqn:StarDelta}
\end{align}
Let $T_m:=\frac{B_m}{D_m^2C}r_m^4\left(2^{m+1}\right)^2 + \frac{B_m}{C}r_m^2\left(2^{m+1}\right)^2K^2$. The function $t\mapsto \sqrt{t}$ is a differentiable concave function, meaning for any $t_1,t_2>0$, $\sqrt{t_2}\leq\sqrt{t_1}+\frac{1}{2\sqrt{t_1}}(t_2-t_1)$. We therefore have
\[\sqrt{t}\leq \sqrt{T_m} +\frac{1}{2\sqrt{T_m}}\left(t-T_m\right). \]
Applying both these to \eqref{eqn:StarDelta} and rearranging, we get
\[t\leq T_m\left(1+ \frac{2B_mr_m^2\left(2^{m+1}\right)^2\sqrt{\log K}}{2(\sqrt{2}-1)C\sqrt{T_m} - B_mr_m^2\left(2^{m+1}\right)^2\sqrt{\log K}} \right). \]
Note that for small enough $\delta,$ the first term in the definition of $T_m$ dominates the second term, i.e., there exists $\delta^{(1)}_0>0$ such that $\forall \delta<\delta^{(1)}_0,$
\begin{eqnarray}
\frac{B_m}{C}r_m^2\left(2^{m+1}\right)^2K^2 &\leq& \frac{B_m}{D_m^2C}r_m^4\left(2^{m+1}\right)^2,\nonumber\\
\Rightarrow r_m^2 \geq K^2D_m^2.\label{eqn:rmSqLargerKSqDSq}
\end{eqnarray}
This means that $T_m\leq 2\frac{B_m}{D_m^2C}r_m^4\left(2^{m+1}\right)^2,$ and hence,
\begin{align*}
 t &\leq  2\frac{B_mr_m^4\left(2^{m+1}\right)^2}{D_m^2C}\left(1+\frac{2B_mr_m^2\left(2^{m+1}\right)^2\sqrt{\log K}}{2(\sqrt{2}-1)C\sqrt{\frac{B_mr_m^4\left(2^{m+1}\right)^2}{D_m^2C}}-B_mr_m^2\left(2^{m+1}\right)^2\sqrt{\log K}} \right)\\
&=2 \frac{B_mr_m^4\left(2^{m+1}\right)^2}{D_m^2C}\left(1+\frac{2D_m\sqrt{B_m\left(2^{m+1}\right)^2{\log K}}}{2(\sqrt{2}-1)\sqrt{C}-D_m\sqrt{B_m\left(2^{m+1}\right)^2{\log K}}} \right). 
\end{align*}
We note here the following lower bound on $B_m$.
\begin{align*}
B_m &= \min\limits_{w\in \mathcal{P}_K}\max\limits_{x,x'\in \mX_m, x\neq x'} \norm{x-x'}^2_{W^{-1}}\\
&\geq \min\limits_{w\in \mathcal{P}_K}\max\limits_{x,x'\in \mX_m, x\neq x'}\lambda_{min}(W^{-1})\norm{x-x'}_2^2\\
&=\min\limits_{w\in \mathcal{P}_K}\max\limits_{x,x'\in \mX_m, x\neq x'}\frac{1}{\lambda_{max}(W)}\norm{x-x'}_2^2\\
&\geq \min\limits_{w\in \mathcal{P}_K}\max\limits_{x,x'\in \mX_m, x\neq x'}\norm{x-x'}_2^2 \\
&= \max\limits_{x,x'\in \mX_m, x\neq x'}\norm{x-x'}_2^2.
\end{align*}

By using the value of $D_m$ as given in Step~6 of the algorithm, we note that 
\begin{align*}
D_m\sqrt{B_m\left(2^{m+1}\right)^2{\log K}}&= 2(\sqrt{2}-1)\sqrt{\frac{{C}}{{\max\limits_{x,x'\in \mX_m, x\neq x'}\norm{x-x'}_2^2\log K}}}\sqrt{B_m\left(2^{m+1}\right)^2{\log K}}\\
&\geq 2(\sqrt{2}-1)\sqrt{\frac{{C}}{{\max\limits_{x,x'\in \mX_m, x\neq x'}\norm{x-x'}_2^2\log K}}}\sqrt{\max\limits_{x,x'\in \mX_m, x\neq x'}\norm{x-x'}_2^2\left(2^{m+1}\right)^2{\log K}}\\
&=\left(2^{m+1}\right).2(\sqrt{2}-1)\sqrt{C}> 2(\sqrt{2}-1)\sqrt{C}.
\end{align*}
Using this we get a bound on $ t $ as:
\begin{align*}
 t &\leq 2 \frac{B_mr_m^4\left(2^{m+1}\right)^2}{D_m^2C}=2 \frac{B_mr_m^4\left(2^{m+1}\right)^2}{4(\sqrt{2}-1)^2C^2}\left(\max\limits_{x,x'\in \mX_m, x\neq x'}\norm{x-x'}_2^2\log K \right)\\
&\leq 2  B_m\left(2^{m+1}\right)^2\left[\frac{r_m^4\log K}{(\sqrt{2}-1)^2C^2} \right].
\end{align*}
Since, by assumption, $C\equiv\lambda_{min}\left(\sum\limits_{k=1}^K x_kx_k^T\right) = \Theta(1)$, we have $t\leq\lim  O\left(B_m\left(2^{m+1}\right)^2r_m^4\log K\right),~\forall\delta<\delta^{(1)}_0.$

\subsection*{Case 2. $\epsilon_m = \left(\frac{1}{2}\right)^{m+1}$.}

We have in this case that, $\frac{D_m\sqrt{C}}{r_m} \geq 1$. Applying the value of $\epsilon_m$ in eq. (\ref{eq: phase length general equation}), we obtain
\begin{align}
t-\frac{B_m}{(\sqrt{2}-1)\epsilon_m^2}D_m^2\sqrt{\log K}\sqrt{t} &\leq \frac{B_m}{\epsilon_m^2}r_m^2 + \frac{B_m}{\epsilon_m^2}K^2D_m^2\\
\Rightarrow   t-\frac{B_m}{(\sqrt{2}-1)}D_m^2\left(2^{m+1}\right)^2\sqrt{\log K}\sqrt{t} &\leq {B_m}r_m^2\left(2^{m+1}\right)^2 + {B_m}\left(2^{m+1}\right)^2K^2D_m^2.\label{eqn:StarDelta2}
\end{align}
Let $T_m:={B_m}r_m^2\left(2^{m+1}\right)^2 + {B_m}\left(2^{m+1}\right)^2K^2D_m^2.$. As before, noting that $t\mapsto \sqrt{t}$ is a concave, differentiable function, we have
\[\sqrt{t}\leq \sqrt{T_m} +\frac{1}{2\sqrt{T_m}}\left(t-T_m\right). \]
Applying this to \eqref{eqn:StarDelta2} and rearranging, we get
\[t\leq T_m\left(1+ \frac{2B_mr_m^2\left(2^{m+1}\right)^2\sqrt{\log K}}{2(\sqrt{2}-1)C\sqrt{T_m} - B_mr_m^2\left(2^{m+1}\right)^2\sqrt{\log K}} \right). \]
Going along the same lines as Case~1, we see that there exists $\delta^{(2)}_0>0$ such that $\forall~\delta<\delta^{(2)}_0$, $T_m\leq 2{B_m}r_m^2\left(2^{m+1}\right)^2$, whence
%We now take the limit $\delta\to 0\Rightarrow  \delta_m\to 0$, to get 
\begin{align*}
 t \leq 2 B_m\left(2^{m+1}\right)^2r_m^2.
\end{align*}
We now set $\delta_0 = \min\{\delta^{(1)}_0,\delta^{(2)}_0\}$.

\end{proof}
% % % % % % % % % % % % % % % % % % % % % % % % % % % % % % % % % % % % % % % %
\section{Justification of elimination criteria}\label{sec:justificationOfEliminationCriteria}
In this section, we argue that progress is made after every phase of the algorithm. We will also show the correctness of the algorithm. Let us define a few terms which will be useful for analysis.
\par  Let $ \mS_m:=\left\{x\in \mX: {\theta^*}^T\left(x^*-x\right)< \frac{1}{2^{m}} \right\} $.  Let $B_m^*:=\min\limits_{w\in \mathcal{P}_K}\max\limits_{(x,x')\in {\color{blue}\mS_m^2}, x\neq x'}\norm{x-x'}^2_{W^{-1}}$, where $W=\sum\limits_{k=1}^{K}w^kx_kx_k$. Finally, define $T_m^*:=\frac{B_m^*}{D_m^2C}r_m^4\left(2^{m+1}\right)^2 + \frac{B_m^*}{C}r_m^2\left(2^{m+1}\right)^2K^2D_m^2$.
\par Define a sequence of favorable events $\left\{\mathcal{G}_m\right\}_{m\geq 1}$ as, \[\mathcal{G}_m:= \left\{ N_m\leq T_m^*\left(1+ \frac{2B_m^*r_m^2\left(2^{m+1}\right)^2\sqrt{\log K}}{2(\sqrt{2}-1)C\sqrt{T_m^*} - B_m^*r_m^2\left(2^{m+1}\right)^2\sqrt{\log K}} \right) \right\}\bigcap \left\{x^*\in \mathcal{X}_{m+1} \right\}\bigcap \left\{\mX_{m+1}\subseteq \mS_{m+1} \right\}.\]
\begin{remark}\label{remark: relation between N_m}
	Conditioned on the event $\mathcal{G}_{m-1}$, $x^*\in \mX_m$ and $\mX_m\subseteq \mS_m$.  Hence,  $B_m \leq B_m^*$ and $T_m\leq T_m^*$. Hence, under the event $\mathcal{G}_{m-1}$,
	\[N_m\leq T_m^*\left(1+ \frac{2B_m^*r_m^2\left(2^{m+1}\right)^2\sqrt{\log K}}{2(\sqrt{2}-1)C\sqrt{T_m^*} - B_m^*r_m^2\left(2^{m+1}\right)^2\sqrt{\log K}} \right) a.s. \]
	Note here that the right hand side is a non-random quantity. 
\end{remark}

\begin{lemma}\label{lemma: probability of good event}
	$ \prob{\mathcal{G}_m\given \mathcal{G}_{m-1},\ldots,\mathcal{G}_1} \geq 1-\delta_m $.
\end{lemma}
\begin{proof}[Proof of lemma \ref{lemma: probability of good event}]
	Let $y = x_i-x_j$ for some $x_i,x_j\in \mX_m, x_i\neq x_j$.  Since $\hat{\theta}_m$ is a least squares estimate of $\theta^*$, conditioned on the realization of the set $\mX_m$, $y^T\left(\hat{\theta}_m-\theta^*\right)$ is a $\norm{y}^2_{(V_{N_m}^m)^{-1}}-$sub-Gaussian random variable. 
	\par By the key lemma \ref{lemma:key lemma on uncertainity after phase m} we have that $\norm{y}^2_{(V_{N_m}^m)^{-1}} \leq \frac{1}{8\left(2^{m+1}\right)^2\log\left(K^2/\delta_m\right)}$. Using property of sub-Gaussian random variables, we write for any $\eta\in \left(0,1\right)$,
	\[\prob{\abs{y^T\left(\hat{\theta}_m-\theta^*\right)} > \sqrt{2\norm{y}^2_{(V_{N_m}^m)^{-1}} \log\left(2/\eta\right) }\given \mathcal{G}_{m-1},\ldots,\mathcal{G}_1 } \leq \eta, \]
	which implies that 
	\[\prob{\abs{y^T\left(\hat{\theta}_m-\theta^*\right)} > \sqrt{\frac{2 \log\left(2/\eta\right)} {8\left(2^{m+1}\right)^2\log\left(K^2/\delta_m\right)} }\given \mathcal{G}_{m-1},\ldots,\mathcal{G}_1 } \leq \eta. \]
	Taking intersection over all possible $y\in \mY\left(\mX_m\right)$, and setting $\eta=2\delta_m/{K^2}$, gives
	\begin{equation}\label{eq:concentration equation}
	\prob{\forall y\in \mY\left(\mX_m\right): \abs{y^T\left(\theta^*-\hat{\theta}_m \right)} \leq 2^{-(m+2)} \given \mathcal{G}_{m-1},\ldots,\mathcal{G}_1 } > 1- \delta_m.
	\end{equation}
	Conditioned on $\mathcal{G}_{m-1}$, $x^*\in\mX_m $. Let $x'\in \mX_m$ be such that $x'\notin \mS_{m+1}$. Let $y=\left(x^*-x'\right)$. Then $y\in \mY\left(\mX_m\right)$.  By eq. (\ref{eq:concentration equation}) we have with probability $\geq 1-\delta_m$:
	\[\left(x^*-x'\right)^T\left(\theta^*-\hat{\theta}_m\right) \leq 2^{-(m+2)}\Rightarrow  \hat{\theta}_m^T\left(x^*-x'\right) > 2^{-(m+1)} -2^{-(m+2)} = 2^{-(m+2)}. \]
	Thus arm $x'$ will get eliminated after phase $m$ by the elimination criteria of algorithm \ref{alg:PEPEG-mixed}(see step 25 of algorithm \ref{alg:PEPEG-mixed}). Hence $\mX_{m+1}\subseteq\mS_{m+1}$ w.p. $\geq 1-\delta_m$. 
	\newline Next, we show that conditioned on $\mathcal{G}_{m-1}$, $x^*\in \mX_{m+1}$, w.p. $\geq 1-\delta_m$. Suppose that $x^*$ gets eliminated at the end of phase $m$. This means that $\exists x'\in \mX_m$, such that $\hat{\theta}_m^T\left(x'-x^* \right) > 2^{-(m+2)}$. However, by eq. \ref{eq:concentration equation},
	\[\left(x'-x^* \right)^T\left(\hat{\theta}_m -\theta^*\right) \leq 2^{-(m+2)}\Rightarrow  {\theta^*}^T\left(x^*-x'\right) <0  \] 
	which is a contradiction. This, along with note \ref{remark: relation between N_m} shows that $\prob{\mathcal{G}_m\given \mathcal{G}_{m-1},\ldots,\mathcal{G}_1} \geq 1-\delta_m$.
	
\end{proof}

\begin{corollary}\label{corollary:correctness}
	\[\prob{\bigcap\limits_{m\geq 1} \mathcal{G}_m}\geq \prod_{m=1}^{\infty}\left(1-\frac{\delta}{m^2}\right)  \geq 1-\delta.\]
\end{corollary}
\begin{corollary}\label{corollary:finite number of phases}
	The maximum number of phases of Algorithm~\ref{alg:PEPEG-mixed} is bounded by $\log_2 \frac{1}{\Delta_{min}}$.
\end{corollary}
\begin{proof}
    Recall that $\Delta_{min}=\min_{x\in\mathcal{X}:x\neq x^*}{\theta^*}^T\left(x^*-x\right).$
	The proof follows by observing that after any phase $m$, under the favorable event $\mathcal{G}_{m-1}$, $\mX_m\subseteq\mS_m$. Since the size $\mS_m$ shrinks exponentially with the number of phases $\left(\text{because}~\mS_m=\left\{ x\in\mX: {\theta^*}^T\left(x^*-x\right)<\frac{1}{2^{m}}\right\}\right)$, we have the result.
\end{proof}

\section{Proof of bound on sample complexity}\label{appendix:proofOfSampleCmpltyBound}
We begin by observing the following useful result from \cite{jamieson-etal19transductive-linear-bandits}. Recall that $$D_{\theta^*}=\max\limits_{w\in \Delta_K} \min\limits_{x\in \mX, x\neq x^*}\frac{\left({\theta^*}^T\left(x^*-x\right)\right)^2}{\norm{x^*-x}^2_{W^{-1}}}$$
\begin{proposition}[\cite{jamieson-etal19transductive-linear-bandits}]\label{propsition:Jamieson bound on B_m}
	\[\sum_{m=1}^{\log_2\frac{1}{\Delta_{min}}} \left(2^{m}\right)^2 B_m^* \leq \frac{4\log_2\left(1/\Delta_{min}\right)}{D_{\theta^*}} .\]
\end{proposition}

Using proposition \ref{propsition:Jamieson bound on B_m} we now give a bound on the asymptotic sample complexity of algorithm \ref{alg:PEPEG-mixed}.
\begin{theorem}%\label{theorem:sample complexity}
 With probability at least $1-\delta$, PEPEG returns the optimal arm after $\tau$ rounds, with
	\begin{equation*}
	\begin{split}
	\tau \leq \left(2048\frac{\log_2\left(1/\Delta_{min}\right)}{D_{\theta^*}}\left[\frac{\left(\log\left(\left(\log_2\left(1/\Delta_{min}\right)\right)^2K^2/\delta\right)\right)^2\log K}{(\sqrt{2}-1)^2C^2} \right] \right)+ \\ \left(256\frac{\log_2\left(1/\Delta_{min}\right)}{D_{\theta^*}}\log\left(\left(\log_2\left(1/\Delta_{min}\right)\right)^2K^2/\delta\right)\right).
	\end{split}
	\end{equation*}
\end{theorem}
%\avim{RAGE requires $\mathcal{O}\left(\frac{1}{D_{\theta^*}}\log{1/\Delta_{min}}\log{\left(\frac{K^2\log^2{1/\Delta_{min}}}{\delta}\right)}+r(\epsilon)\log_2{1/\Delta_{min}}\right)$ samples. How does this compare with our bound?}
\begin{proof}
	The proof follows from Lemma~\ref{lemma:phaseLengthBound} (phase length bound), Corollary~\ref{corollary:finite number of phases} (bound on number of phases), Prop.~\ref{propsition:Jamieson bound on B_m} above and the fact that the sum of several non negative quantities is bigger than their max. 
	
	To begin with, the discussion in Sec.~\ref{sec:justificationOfEliminationCriteria} shows that in every phase, $B_m\leq B^*_m.$ Next, Lemma~\ref{lemma:phaseLengthBound} gives us (w.h.p),
	\begin{align*}
	   \tau &= \sum_{m=1}^{\log_2\left(1/\Delta_{min}\right)}N_m\\
	         &\leq \sum_{m=1}^{\log_2\left(1/\Delta_{min}\right)}\max\left\lbrace 2 B_m^*\left(2^{m+1}\right)^2\left[\frac{r_m^4\log K}{(\sqrt{2}-1)^2C^2} \right],2B_m^*\left(2^{m+1}\right)^2r_m^2\right\rbrace\\
	         &\leq \sum\limits_{m=1}^{\log_2\left(1/\Delta_{min}\right)}  2 B_m^*\left(2^{m+1}\right)^2\left[\frac{r_m^4\log K}{(\sqrt{2}-1)^2C^2} \right] +  \sum\limits_{m=1}^{\log_2\left(1/\Delta_{min}\right)} 2 B_m^*\left(2^{m+1}\right)^2r_m^2
	\end{align*}
	Hence, using the fact that $r_m=\sqrt{8\log{K}^2/\delta_m}$ and invoking Prop.~\ref{propsition:Jamieson bound on B_m} we get
	\begin{align*}
	\begin{split}
 \tau &\leq \sum\limits_{m=1}^{\log_2\left(1/\Delta_{min}\right)}  512B_m^*\left(2^{m}\right)^2\left[\frac{\left(\log\left(K^2/\delta_m\right)\right)^2\log K}{(\sqrt{2}-1)^2C^2} \right]\\  & \quad \quad \quad+ \sum\limits_{m=1}^{\log_2\left(1/\Delta_{min}\right)}       64 B_m^*\left(2^{m+1}\right)^2\log\left(K^2/\delta_m\right)\\
	&\stackrel{(\ast)}{\leq}  2048\frac{\log_2\left(1/\Delta_{min}\right)}{D_{\theta^*}}\left[\frac{\left(\log\left(\left(\log_2\left(1/\Delta_{min}\right)\right)^2K^2/\delta\right)\right)^2\log K}{(\sqrt{2}-1)^2C^2} \right]\\& \quad \quad \quad+       256\frac{\log_2\left(1/\Delta_{min}\right)}{D_{\theta^*}} \log\left(\left(\log_2\left(1/\Delta_{min}\right)\right)^2K^2/\delta\right),
	\end{split}
	\end{align*}
	where $(\ast)$ follows from the fact that $\frac{K^2}{\delta_m}=\frac{m^2K^2}{\delta}\leq\frac{\left(\log_2\left(1/\Delta_{min}\right)\right)^2K^2}{\delta}.$
\end{proof}

\section{Experiment Details}
In this section, we provide some details on the implementation of each algorithm. Each experiment
was repeated 50 times and the errorbar plots show the mean sample complexity with 1-standard deviations. 
\begin{itemize}
\item For implementation of PELEG, as mentioned in Sec. \ref*{Experiments}, we ignore the intersection with the ball $B(0,D_m)$ in the phase stopping criterion. This helps in implementing a closed form expression for the stopping rule. The learning rate parameter in the EXP-WTS subroutine is set to be equal to $(1/D_m^2)\sqrt{8\log K/t}$.  
\item LinGapE: In the paper of \citep{XuAISTATS} LinGapE was simulated using a greedy arm
selection strategy that deviates from the algorithm that is analyzed. We instead implement the
LinGapE algorithm in the form that it is analyzed.
\item For implementation of RAGE, ALBA and $\mathcal{XY-}$ORACLE, we have used the code provided in the Supplementary material of Fiez et al \cite{jamieson-etal19transductive-linear-bandits}. We refer the readers to Appendix Sec. F of \cite{jamieson-etal19transductive-linear-bandits} for further details of their implementations.
\end{itemize}
\end{document}